\documentclass[sigconf]{acmart}

\usepackage{balance}

\usepackage{wrapfig}

\usepackage{caption}
\usepackage[separate-uncertainty=true]{siunitx}
\usepackage{mathtools}
\usepackage{graphicx}
\usepackage{courier}
\usepackage{epstopdf}
\usepackage{color}
\usepackage{csquotes}

\usepackage{amsmath}
\usepackage{amsfonts}
\usepackage{amssymb}
\usepackage[subnum]{cases}
\usepackage{booktabs} %For formal tables
\usepackage{pifont}

\usepackage{algorithm,algorithmicx,algpseudocode}

\usepackage{amssymb}% http://ctan.org/pkg/amssymb
\usepackage{pifont}% http://ctan.org/pkg/pifont

\usepackage {tikz}
\usetikzlibrary {positioning}
\definecolor {processblue}{cmyk}{0.96,0,0,0}
\usepackage{lipsum,adjustbox}

\usepackage[subnum]{cases}
\usepackage{color,soul}

\usepackage{subcaption}

\usepackage{footnote}
\makesavenoteenv{tabular} %ICDE version does not support

\usepackage{multirow}

\usepackage{booktabs,subcaption,amsfonts,dcolumn}
\newcolumntype{d}[1]{D..{#1}}
\AtBeginDocument{%
  \providecommand\BibTeX{{%
    \normalfont B\kern-0.5em{\scshape i\kern-0.25em b}\kern-0.8em\TeX}}}

\begin{document}

\title[Robust Hierarchical Graph Classification with Subgraph Attention]{Robust Hierarchical Graph Classification with \\Subgraph Attention}

% \affiliation{%

\author{Sambaran Bandyopadhyay}
\affiliation{%
  \institution{IBM Research AI \& Indian Institute of Science}
  \city{Bangalore, India}}
\email{samb.bandyo@gmail.com}

\author{Manasvi Aggarwal}
\affiliation{%
  \institution{Indian Institute of Science}
  \city{Bangalore, India}}
\email{manasvia@iisc.ac.in}

\author{M. Narasimha Murty}
\affiliation{%
  \institution{Indian Institute of Science}
  \city{Bangalore, India}}
\email{mnm@iisc.ac.in}

% \affiliation{%

% \affiliation{%



% %%
% %% By default, the full list of authors will be used in the page
% %% headers. Often, this list is too long, and will overlap
% %% other information printed in the page headers. This command allows
% %% the author to define a more concise list
% %% of authors' names for this purpose.

\begin{abstract}
Graph neural networks get significant attention for graph representation and classification in machine learning community. Attention mechanism applied on the neighborhood of a node improves the performance of graph neural networks. Typically, it helps to identify a neighbor node which plays more important role to determine the label of the node under consideration. But in real world scenarios, a particular subset of nodes together, but not the individual pairs in the subset, may be important to determine the label of the graph. To address this problem, we introduce the concept of subgraph attention for graphs. On the other hand, hierarchical graph pooling has been shown to be promising in recent literature. But due to noisy hierarchical structure of real world graphs, not all the hierarchies of a graph play equal role for graph classification. Towards this end, we propose a graph classification algorithm called SubGattPool which jointly learns the subgraph attention and employs two different types of hierarchical attention mechanisms to find the important nodes in a hierarchy and the importance of individual hierarchies in a graph. Experimental evaluation with different types of graph classification algorithms shows that SubGattPool is able to improve the state-of-the-art or remains competitive on multiple publicly available graph classification datasets. We conduct further experiments on both synthetic and real world graph datasets to justify the usefulness of different components of SubGattPool and to show its consistent performance on other downstream tasks.
\end{abstract}




\maketitle

\section{Introduction}\label{sec:intro}
Graphs are the most suitable way to represent different types of relational data such as social networks, protein interactions and molecular structures. Typically, A graph is represented by $G = (V,E)$, where $V$ is the set of nodes and $E$ is the set of edges. Further, each node $v_i \in V$ is also associated with an attribute (or feature) vector $x_i \in \mathbb{R}^D$.
Recent advent of deep representation learning has heavily influenced the field of graphs. Graph neural networks (GNNs) \cite{defferrard2016convolutional,xu2018how} are developed to use the underlying graph as a computational graph and aggregate node attributes from the neighbors of a node to generate the node embeddings \cite{kipf2016semi}. A simplistic message passing framework \cite{gilmer2017neural} for graph neural networks can be presented by the following equations. 
\begin{flalign}\label{eq:generalGNN}
    &h_v^l = COMBINE^l\Bigg( \Big\{ h_v^{l-1}, AGGREGATE^k \big( \{h_{v'}^{l-1} : v' \in \mathcal{N}_G(v) \} \big) \Big\}\Bigg) \nonumber \\
    &\;\;\;\;\;\;\;\;\;\;\;x^{G} = READOUT\Big( \Big\{ h_v^L : v \in V(G) \Big\} \Big)
\end{flalign}
Here, $h_v^l$ is the representation of node $v$ of graph $G$ in $l$-th layer of the GNN. The function $AGGREGATE$ considers representation of the neighboring nodes of $v$ from the $(l-1)$th layer of the GNN and maps them into a single vector representation. As neighbors of a node do not have any ordering in a graph and the number of neighbors can vary for different nodes, $AGGREGATE$ function needs to be permutation invariant and should be able to handle different number of nodes as input. Then, $COMBINE$ function uses the node representation of $v$th node from $(l-1)$the layer of GNN and the aggregated information from the neighbors to obtain an updated representation of the node $v$. Finally for the graph level tasks, $READOUT$ function (also known as graph pooling) generates a summary representation $x^G$ for the whole graph $G$ from all the node representations $h_v^L$, $\forall v \in V(G)$ from the final layer (L) of GNN. Similar to $AGGREGATE$, the function $READOUT$ also needs to be invariant to different node permutations of the input graph, and should be able to handle graphs with different number of nodes.

In the existing literature, different types of neural architectures are proposed to implement each of the three functions mentioned in Equation \ref{eq:generalGNN}. For example, GraphSAGE \cite{hamilton2017inductive} implements 3 different variants of the $AGGREGATE$ function with mean, maxpool and LSTM respectively. 
For a graph level task such as graph classification \cite{xu2018how,duvenaud2015convolutional}, GNNs jointly derive the node embeddings and use different pooling mechanisms \cite{ying2018hierarchical,lee2019self} to obtain a representation of the entire graph.
Recently, attention mechanisms on graphs show promising results for both node classification \cite{velivckovic2017graph} and graph classification \cite{lee2019self,lee2018graph} tasks. There are different ways to compute attention mechanisms on graph. \cite{velivckovic2017graph} compute attention between a pair of nodes in the immediate neighborhood to capture the importance of a node on the embedding of the other node by learning an attention vector. \cite{lee2018graph} compute attention between a pair of nodes in the neighborhood to guide the direction of a random walk in the graph for graph classification. 
\cite{lee2019self} propose self attention pooling of the nodes which is then used to capture the importance of the node to generate the label of the entire graph.

Most of the attention mechanisms developed in graph literature use attention to derive the importance of a node or a pair of nodes for different tasks. But in real world situation, calculating importance up to a pair of nodes is not adequate.
In molecular biology or in social networks, the presence of particular sub-structures (a subset of nodes with their connections and features), potentially of varying sizes, in a graph often determines its label. Hence, all the nodes collectively in such a substructure are important, and they may not be important individually or in pairs to classify the graph. 
In Figure \ref{fig:exam_graph}, each node (indexed from $a$ to $g$) in the small synthetic graph can be considered as an agent  whose attributes determine its opinion (1:positive, 0: neutral, -1: negative) about 4 products. Suppose the graph can be labelled +1 only if there is a subset of connected (by edges) agents who jointly have positive opinion about all the product. In this case, the blue shaded connected subgraph $(a,b,c)$ is important to determine the label of the graph. Please note, attention over the pairs \cite{velivckovic2017graph} is not enough as $(a,b)$ cannot make the label of the graph +1 by itself. Also, multiple layers of graph convolution \cite{kipf2016semi} with pair-wise attention may not work as the aggregated features of a node get corrupted after the feature aggregation by the first few convolution layers.
Besides, recent literature also shows that higher order GNNs that directly aggregate features from higher order neighborhood of a node are theoretically more powerful than 1st order GNNs \cite{morris2019weisfeiler}.
With these motivations, we develop a novel higher order attention mechanism in the graph which operates in the subgraph level in the vicinity of a node. We call it subgraph attention mechanism and use it for graph classification\footnote{Subgraph attention can easily be applied for node classification as well. But we focus only on graph classification in this paper.}.

% \textbf{Node classification}: Given a graph $G=(V,E)$ with each node $v_i$ associated with an attribute vector $x_i \in \mathbb{R}^D$, and a subset of nodes $V_s \subseteq V$ with each node $v_i \in V_s$ labelled with $y_i \in \mathcal{L}_{n}$ (set of discrete labels for the nodes of the graph), the task is to predict the label of a node $v_j \in V_u = V \setminus V_s$ using the structure and the node attributes of the entire graph and the node labels from $V_s$.
% %and the node attributes. 
% Essentially, this leads to learning a function $f_{n} : V \mapsto \mathcal{L}_{n}$ for the given graph $G$. 
% %Here, $\mathcal{L}_{n}$ is the set of discrete labels for the nodes of the graph.

On the other hand, different types of graph pooling (i.e., $READOUT$ function in Equation \ref{eq:generalGNN}) mechanisms \cite{duvenaud2015convolutional,gilmer2017neural,morris2019weisfeiler} have been proposed in the recent GNN literature. Simple functions such as taking sum or mean of all the node representations to compute the graph-level representation are studied in \cite{duvenaud2015convolutional}.
Recently, hierarchical graph pooling \cite{ying2018hierarchical,morris2019weisfeiler} gains significant interest as it is able to capture the intrinsic hierarchical structure of several real-world graphs. For e.g., in a social network, one must model both the ego-networks around individual nodes, as well as the coarse-grained relationships between entire communities \cite{newman2003structure}. Instead of directly obtaining a graph level summary vector, hierarchical pooling mechanisms recursively converts the input graph to graphs with smaller sizes. But hierarchical representation \cite{ying2018hierarchical} often fails to perform well in practice mainly due to two major shortcomings. First, there is significant loss of information in learning the sequential hierarchies of a graph when the data is limited. Second, it treats all the nodes within a hierarchy, and all the hierarchies equally while computing the entire graph representation. But for some real-world graphs, the structure between the sub-communities may be more important than that between the nodes or the communities to determine the label of the entire graph \cite{newman2003structure}. Moreover, due to presence of noise, some of the discovered hierarchies may not follow the actual hierarchical structure of the graph \cite{sun2017breaking}, and can negatively impact the overall graph representation. To address these issues, we again use attention to differentiate different units of a hierarchical graph representation in a GNN framework. Thus, our contributions in this paper are multifold, as follows:
\begin{itemize}
    \item We propose a novel higher order attention mechanism (called \textit{subgraph attention}) for graph neural networks, which is based on the importance of a subgraph of dynamic size to determine the label of the graph.
    \item We also propose hierarchical attentions in graph representation. More precisely, we propose \textit{intra-level} and \textit{inter level} attention which respectively find important nodes within a hierarchy and important hierarchies of the hierarchical representation of the graph. This enables the overall architecture to minimize the loss of information in the hierarchical learning and to achieve robust performance on real world noisy graph datasets.
    \item We propose a novel neural network architecture \textit{SubGattPool} (\textul{Sub}-\textul{G}raph \textul{att}ention network with hierarchically attentive graph \textul{Pool}ing) to combine the above two ideas for graph classification. Thorough experimentation on both real world and synthetic graphs shows the merit of the proposed algorithms over the state-of-the-arts. %The source code of SubGattPool is made available at \url{https://bit.ly/38kRVxj} to ease the reproducibility of the results.
\end{itemize}

\begin{figure}
    \centering
    \includegraphics[width=0.7\columnwidth]{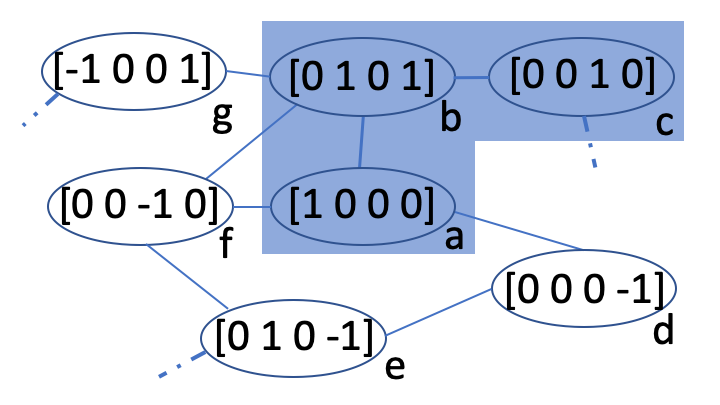}
    \caption{Example to motivate subgraph attention}
  \label{fig:exam_graph}
\end{figure}

\section{Related Work and the Research Gaps}\label{sec:related}
A survey on network representation \cite{grover2016node2vec,bandyopadhyay2018fscnmf} learning and graph neural networks can be found in \cite{wu2019comprehensive}. For the interest of space, we briefly discuss some more prominent approaches for graph classification and representation. Graph kernel based approaches \cite{vishwanathan2010graph}, which map the graphs to Hilbert space implicitly or explicitly, remain to be the state-of-the-art for graph classification for long time. There are different types of graph kernels present in the literature, such as
random walk based kernel \cite{kashima2003marginalized}, 
shortest path based kernels \cite{borgwardt2005shortest},
graphlet counting based kernel \cite{shervashidze2009efficient},
%and most prominently the 
Weisfeiler-Lehman subtree kernel \cite{shervashidze2011weisfeiler} and Deep graph kernel \cite{yanardag2015deep}. But most of the existing graph kernels use hand-crafted features and they often fail to adapt the data distribution of the graph.

Significant progress happened in the domain of node representation and node level tasks via graph neural networks. Spectral graph convolutional neural networks with fast localized convolutions \cite{defferrard2016convolutional,kipf2016semi}, graph attention (GAT) over a pair of connected node in the graph convolution framework \cite{velivckovic2017graph}, attention over different layers of convolution \cite{xu2018representation},
position aware graph neural networks \cite{you2019position} and hyperbolic graph convolution networks \cite{chami2019hyperbolic}
are some notable examples of GNN for node representation.
To go from node embeddings to a single representation for the whole graph, simple aggregation technique such as taking the average of node embeddings in the final layer of a GCN \cite{duvenaud2015convolutional} and more advanced deep learning architectures that operate over the sets \cite{gilmer2017neural}
%and \cite{zhang2018end} 
have been used. 
%Hierarchical neural network structures combined with different graph pooling strategies \cite{zhang2018end} are also proposed in the literature.
Attention based graph classification technique GAM \cite{lee2018graph} is proposed, which processes only a portion of the graph by adaptively selecting a sequence of informative nodes.
DIFFPOOL \cite{ying2018hierarchical} is a recently proposed hierarchical GNN which uses a GCN based pooling to create a set of hierarchical graphs in each level. \cite{lee2019self} propose a self attention based pooling strategy which determines the importance of a node to find the label of the graph.
%Inspired by CapsNet \cite{sabour2017dynamic}, Capsule GNN (CapsGNN) is proposed \cite{xinyi2018capsule}.
Different extensions of GNNs, such as Ego-CNN \cite{tzeng2019distributed} and ChebyGIN \cite{knyazev2019understanding} are proposed for graph classification.
%\cite{tzeng2019distributed} propose an Ego-CNN architecture which is a feedforward convolutional neural network that can be jointly learned with a supervised task to help identify the task specific critical structures in the graph. 
Theoretical frameworks to analyze the representational power of GNNs are proposed in \cite{xu2018how,maron2019provably}.
\cite{knyazev2019understanding} study the ability of attention GNNs to generalize to larger and complex graphs.

Higher order GNNs which operate beyond immediate neighborhood are proposed recently. Based on higher dimensional Weisfeiler-Leman algorithm, k-GNN \cite{morris2019weisfeiler} is proposed which derive the representation of all the subgraphs of size $k$ through convolution. Mixhop GNN for node classification is proposed in \cite{abu2019mixhop} which aggregates node features according to the higher order adjacency matrices. Though these higher order GNNs are more powerful representation of graphs, they do not employ attention in the higher order neighborhood. To the best of our knowledge, \cite{yang2019spagan} is the only work to propose an attention mechanism on the shortest paths starting from a node to generate the node embedding. However, their computation of shortest path depends on the pairwise node attention and this may fail in the cases when a collection of nodes together is important, but not the individual pairs. Our proposed subgraph attention addresses this gap in the literature.
Further, hierarchical pooling as proposed in DIFFPOOL \cite{ying2018hierarchical} has become a popular pooling strategy in GNNs \cite{morris2019weisfeiler}. But it suffers because of the loss of information and its nature to represent the whole graph by the last level (containing only a single node) of the hierarchy. As discussed in Section \ref{sec:intro}, some intermediate levels may play more important role to determine the label of the entire graph than the last one \cite{newman2003structure}. The intra-level and inter level attention mechanisms proposed in this work precisely address this research gap in hierarchical graph representation.

% As discussed in Section \ref{sec:intro}, attention mechanisms used in the existing GNN literature only consider self-attention or the pair-wise attention between two nodes. Recently, attention over an area for a grid structure data (for e.g., images) is proposed by \cite{DBLP:conf/icml/LiKBS19} and show promising results for tasks like image captioning and machine translation. With this motivation, and to fill the research gap, we propose subgraph attention which generalizes attention to a subgraph level and show its efficiency for node and graph classification in this paper.

\section{Proposed Approach: SubGattPool}\label{sec:SubGattPool}
We formally define the problem of graph classification first.
Given a set of $M$ graphs $\mathcal{G} = \{G_1, G_2, \cdots, G_M\}$, and a subset of graphs $\mathcal{G}_s \subseteq \mathcal{G}$ with each graph $G_i \in \mathcal{G}_s$ is labelled with $Y_i \in \mathcal{L}_{g}$ (the subscript $g$ stands for `graphs'), the task is to predict the label of a graph $G_j \in \mathcal{G}_u = \mathcal{G} \setminus \mathcal{G}_s$ using the structure of the graphs and the node attributes, and the graph labels from $\mathcal{G}_s$. Again, this leads to learning a function $f_{g} : \mathcal{G} \mapsto \mathcal{L}_{g}$. Here, $\mathcal{L}_{g}$ is the set of discrete labels for the graphs.

Figure \ref{fig:SubGattPool} shows the high-level architecture of SubGattPool. One major component of SubGattPool is the generation of node representations through \textul{SubG}raph \textul{att}ention (referred as \textit{SubGatt}) layer. Below, we describe the building blocks of SubGatt for any arbitrary graph. %A \textbf{Sub}-\textbf{G}raph \textbf{att}ention network (referred as SubGatt) can be built by stacking multiple layers of subgraph attention.
For the ease of reading, we summarize all the important notations used in this paper in Table \ref{tab:notations}.

\begin{table}
\centering
\resizebox{\linewidth}{!}{%
\begin{tabular}{*6c}
	\toprule
	\sffamily{Notations} & Explanations\\
%    \sffamily{} & & & & Words & Distribution & links \\
    \hline
	\midrule
    $\mathcal{G}=\{G_1,\cdots,G_M\}$ & Set of graphs in a graph dataset \\
    $G=(V,E,X)$ & One single graph \\
    $\mathcal{L}_{g}$ & Set of discrete labels for graphs\\
    $x_j \in \mathbb{R}^D$ & Attribute vector for $j$th node \\
    $\mathbf{S}_i = \{S_{i1},\cdots,S_{iL}\}$ & Multiset of sampled subgraph for the node $v_i$.\\
    $\hat{x}_{il}$  & Derived feature vector of a subgraph\\
    $T$ & Maximum size (i.e., number of nodes) of a subtree\\ 
    $L$ & Number of subgraphs to sample for each node\\
    $x^G \in \mathbb{R}^K$ & Final representation of the graph $G$ \\
    $G^1,\cdots,G^R$ & Level graphs of some input graph $G$ \\
    $Z_r \in \mathbb{R}^{N_r \times K}$ & Embedding matrix of $G^r$ \\
    $P_r \in \mathbb{R}^{N_r \times N_{r+1}}$ & Node assignment matrix from $G^r$ to $G^{r+1}$ \\
    \bottomrule
	\end{tabular}
	}
\caption{Different notations used in the paper}
\label{tab:notations}
\end{table}

\subsection{Subgraph Attention Mechanism}\label{sec:SubGatt}
The input to the subgraph attention network is an attributed graph $G=(V,E)$, where $V = \{v_1,v_2,\cdots,v_N\}$ is the set of $N$ nodes and $x_i \in \mathbb{R}^D$ is the attribute vector of the node $v_i \in V$. The output of the model is a set of node features (or embeddings) $h_i \in \mathbb{R}^K$, $\forall i \in [N]$ ($K$ is potentially different from $D$). We use $[N]$ to denote the set $\{1,2,\cdots,N\}$ for any positive integer $N$.
We define the immediate (or first order) neighborhood of a node $v_i$ as $\mathcal{N}_i = \{v_j | (v_i,v_j) \in E\}$. For the simplicity of notations, we assume an input graph $G$ to be undirected for the rest of the paper, but extending it for directed graph is straightforward.

\subsubsection{Subgraph selection and Sampling}\label{sec:sub_selec}
For each node in the graph, we aim to find the importance of the nearby subgraphs to that node. In general, subgraphs can be of any shape or size. Motivated by the prior works on graph kernels \cite{shervashidze2011weisfeiler}, we choose to consider only a set of rooted subtrees as the set of candidate subgraphs. So for a node $v_i$, any tree of the form $(v_i)$, or $(v_i, v_j)$ where $(v_i, v_j) \in E$, or $(v_i, v_j, v_k)$ where $(v_i, v_j) \in E$ and $(v_j, v_k) \in E$, and so on will form the set of candidate subgraphs of $v_i$. We restrict that maximum size (i.e., number of nodes) of a subtree is $T$. Also note that, the node $v_i$ is always a part of any candidate subgraph for the node $v_i$ according to our design. For example, all possible subgraphs of maximum size 3 for the node a in Figure \ref{fig:exam_graph} are: (a), (a,b), (a,d), (a,f), (a,b,c), (a,b,f), (a,b,g), (a,d,e), (a,f,e) and (a,f,b).

Depending on the maximum size ($T$) of a rooted subtree, the number of candidate subgraphs for a node can be very large. For example, the number of rooted subgraphs for the node $v_i$ is 
$d_{v_i} \times \sum\limits_{v_j \in \mathcal{N}(v_i)} (d_{v_j}-1) \times \sum\limits_{v_k \in \mathcal{N}(v_j) \setminus \{v_i\}} |\mathcal{N}(v_k) \setminus \{v_i,v_j\}|$, 
where $d_{v}$ is the degree of a node $v$ and $T=4$. Clearly, computing attention over these many subgraphs for each node is computationally difficult. So we employ a subgraph subsampling technique, inspired by the node subsampling techniques for network embedding \cite{hamilton2017inductive}. First, we fix the number of subgraphs to sample for each node. Let the number be $L$. For each node in the input graph, if the total number of rooted subtrees of size $T$ is more than (or equal to) $L$, we randomly sample $L$ number of subtrees without replacement. If the total number of rooted subtrees of size $T$ is less than $L$, we use round robin sampling (i.e., permute all the subtrees, picking up samples from the beginning of the list; after consuming all the trees, again start from the beginning till we complete picking $L$ subtrees). For each node, sample of subtrees remains same for one epoch of the algorithm (explained in the next subsection) and new samples are taken in each epoch. In any epoch, let us use the notation $\mathbf{S}_i = \{S_{i1},\cdots,S_{iL}\}$ to denote the set (more precisely it is a multiset as subgraphs can repeat) of sampled subgraph for the node $v_i$.

\begin{figure*}[h!]
  \centering
    \includegraphics[width=0.8\linewidth]{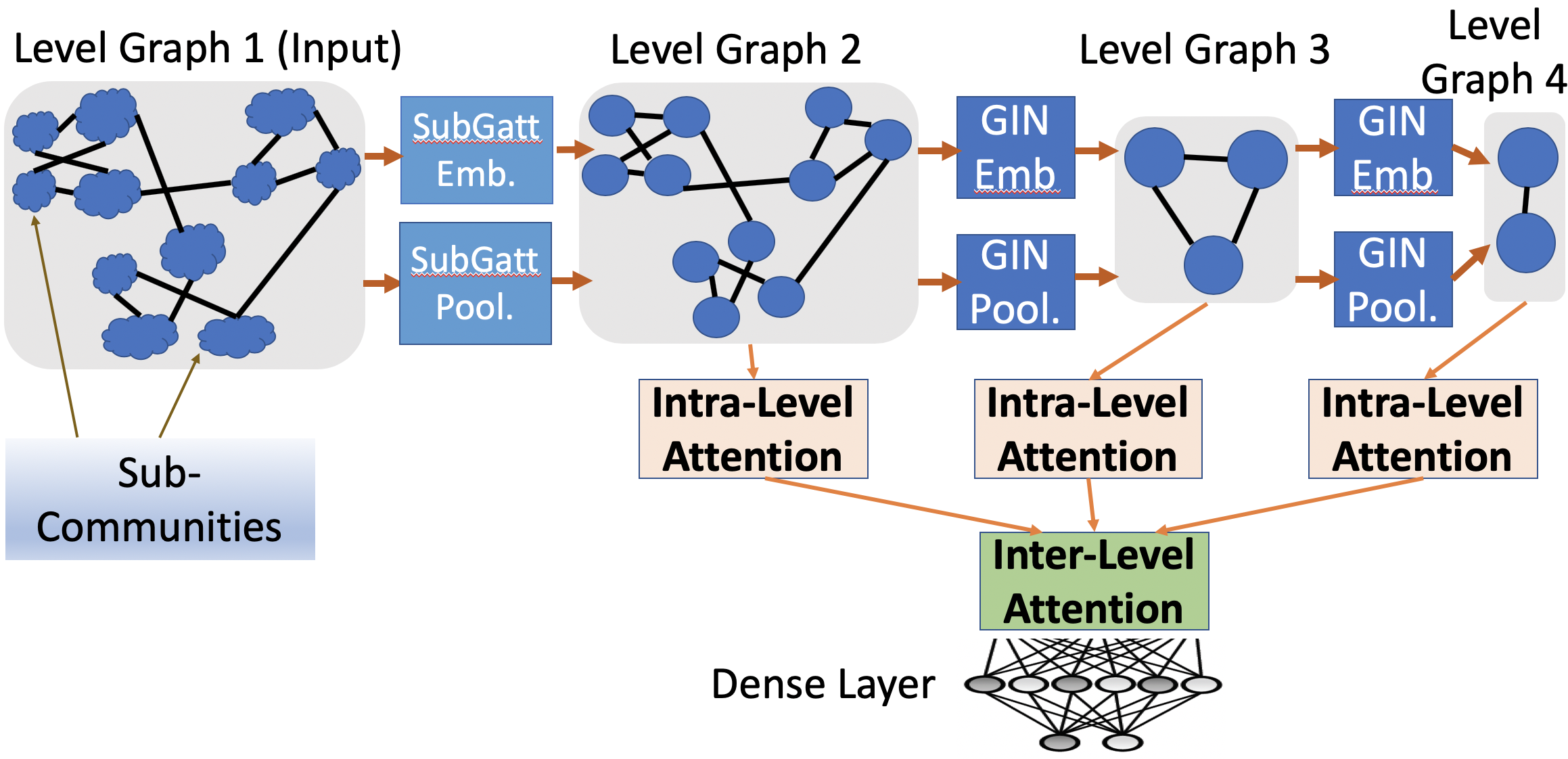}
    \caption{SubGattPool Network for graph classification}
    \label{fig:SubGattPool}
\end{figure*}

\subsubsection{Subgraph Attention Network}\label{sec:sub_atten}
This subsection describes the attention mechanism on the set of rooted subtrees selected for each epoch of the algorithm. As mentioned, the node of interest is always positioned as the root of each subgraph generated for that node. Next step is to generate a feature for the subgraph. We tried different simple feature aggregations (for e.g., mean) of the nodes that belong to the subgraph as the feature of the subgraph. It turns out that concatenation of the features of nodes gives better performance. But for the attention to work, we need equal length feature vectors (the length is $TD$) for all the subgraphs. So if a subgraph has less than $T$ nodes, we append zeros at the end to assign equal length feature vector for all the subgraphs. For example, if the maximum size of a subgraph is $T=4$, then the feature of the subgraph $(v_i,v_j,v_k)$ is $[x_i || x_j || x_k || 0] \in \mathbb{R}^{4D}$, where $||$ is the concatenation operation and $0$ is the zero vector in $\mathbb{R}^D$. Let us denote this derived feature vector of any subgraph $S_{il}$ as $\hat{x}_{il} \in \mathbb{R}^{TD}$, $\forall i \in [N]$ and $\forall l \in [L]$.

Next, we use self-attention on the features for the sampled subgraphs for each node as described here. As the first step, we use a shared linear transformation, parameterized by a trainable weight matrix $W \in \mathbb{R}^{K \times TD}$, to the feature of all the sampled subgraphs $S_{il}$, $\forall i \in [N]$ and $\forall l \in [L]$ selected in an epoch. Next we introduce a trainable self attention vector $a \in \mathbb{R}^{K}$ to compute the attention coefficient $\alpha_{il}$ which captures the importance of the subgraph $S_{il}$ on the node $v_i$, as follows:
\begin{flalign}\label{eq:nodeEmb}
    \alpha_{il} = \frac{exp(\sigma(a^T W \hat{x}_{il}))}{\sum\limits_{l' \in [L]} exp(\sigma(a^T W \hat{x}_{il'}))} \;\;, \nonumber \\
    h_i = \sigma\Big( \sum\limits_{l=1}^L \alpha_{il} W \hat{x}_{il} \Big) \in \mathbb{R}^K \; , \;\;\;\; \forall i \in [N]
\end{flalign}
Here $\sigma()$ is a non-linear activation function. We have used Leaky ReLU as the activation function for all the experiments. $\alpha_{il}$ gives normalized attention scores over the set of sampled subgraphs for each node. We use them to compute the representation $h_i$ of a node $v_i$ as shown in Eq. \ref{eq:nodeEmb}.
% \begin{equation}\label{eq:nodeEmb}
%     h_i = \sigma\Big( \sum\limits_{l=1}^L \alpha_{il} W \hat{x}_{il} \Big) \; , \;\;\;\; \forall i \in [N] 
% \end{equation}
Please note, the attention mechanism described in \cite{velivckovic2017graph} operates only over the immediate neighboring nodes, whereas the higher order attention mechanism proposed in this work operates over the subgraphs. Needless to say, one can easily extend the above subgraph attention by multi-head attention by employing few independent attention mechanisms of Eq. \ref{eq:nodeEmb} and concatenate the resulting representations \cite{vaswani2017attention}. This completes one full subgraph attention layer. We can stack such multiple layers to design a full SubGatt network.

% \subsection{Learning and Node Classification}\label{sec:nodeClassi}
% To use subgraph attention network for node classification, we use a softmax (or logistic sigmoid, depending on the number of classes to predict) as the final layer on top of the node representations. We use standard back propagation algorithm with ADAM optimization to learn the parameters of SubGatt by minimizing the cross entropy loss. For one subgraph attention layer, we just need to learn the linear transformation matrix $W \in \mathbb{R}^{K \times TD}$ and the attention vector $a \in \mathbb{R}^K$. Hence the number of parameters to learn is limited for a smaller value of $T$ (maximum size of a subgraph).

% \textbf{Runtime Complexity of a subgraph attention layer}: First we need to compute the set of possible subgraphs (i.e., rooted subtrees of maximum size $T$) and their respective features for each node, once for the whole dataset. Computing set of all subgraphs and their features for a node takes $d^T$, where $d$ is the average degree of the nodes in the graph. We set $T=4$ for our experiments and average degree of the nodes in real life graphs are small as the networks are highly sparse in nature. Next, in each epoch of SubGatt, we sample $L$ subgraphs for each node and compute Eq. \ref{eq:nodeEmb}. Hence, total runtime for each epoch of subGatt takes $O(NLKTD+Nd^T)$, which is linear with the number of nodes, for a sparse graph. Experimentally, we observe SubGatt to converge fast on all the datasets. 

\subsection{Hierarchically Attentive Graph Pooling}\label{sec:graphClassi}
%As discussed in Section \ref{sec:related}, GNN architectures combined with different pooling mechanisms got promising results for graph classification. The main challenge here is to obtain a single representation of the entire graph which can be used as the features for graph classification in an end-to-end fashion. Global pooling mechanisms, where the node representation are averaged or summed to obtain a graph representation for classification is proposed early in the literature \cite{duvenaud2015convolutional}. Recently, hierarchical graph pooling becomes popular among the researchers, where a graph is converted to a graph of sub-communities in the next level, and then further to a graph of communities and till it becomes only a single node \cite{ying2018hierarchical,lee2019self}. Our approach, though inspired from the hierarchical pooling proposed by \cite{ying2018hierarchical}, combines both global and hierarchical pooling by attention in different levels in the graph. We refer the proposed graph classification architecture by SubGattPool (\textbf{Sub}-\textbf{G}raph and sublevel \textbf{att}ention based \textbf{Pool}ing mechanism), which is described below.

This subsection discusses all the components of SubGattPool architecture. As shown in Figure \ref{fig:SubGattPool}, there are $R=4$ different levels of the graph in the hierarchical architecture. The first level is the input graph. Let us denote these level graphs (i.e., graphs at different levels) by $G^1, \cdots, G^R$. There is a GNN layer between the level graph $G^r$ (i.e., the graph at level $r$) and the level graph $G^{r+1}$. This GNN layer comprises of an embedding layer which generates the embedding of the nodes of $G^r$ and a pooling layer which maps the nodes of $G^r$ to the nodes of $G^{r+1}$. We refer the GNN layer between the level graph $G^r$ and $G^{r+1}$ by $r$th layer of GNN, $\forall r=1,2,\cdots,R-1$.
%The last level graph $G^R$ contains only one node, whose feature summarizes the entire input graph. 
Pleas note, number of nodes $N_1$ in the first level graph depends on the input graph, but we keep the number of nodes $N_r$ in the consequent level graphs $G^r$ ($\forall r=2,\cdots,R$) fixed for all the input graphs (in a graph classification dataset), which help us to design the shared hierarchical attention mechanisms, as discussed later. As pooling mechanisms shrink a graph, $N_r>N_{r+1}$, $\forall r \leq R-1$.

Let us assume that any level graph $G^r$ is defined by its adjacency matrix $A_r \in \mathbb{R}^{N_r \times N_r}$ and the feature matrix $X_r \in \mathbb{R}^{N_r \times K}$ (except for $G^1$, which is the input graph and its feature matrix $X_r \in \mathbb{R}^{N_r \times D}$). 
%Naturally, $A_1$ and $X_1$ are given to the GNN. 
The $r$th embedding layer and the pooling layer are defined by:
\begin{flalign}\label{eq:embPool}
    Z_r = 
    \begin{cases}
        \text{SubGatt}_{embed}(A_r,X_r) \; , \; r=1 \\
        \text{GIN}_{r,embed}(A_r,X_r) \; , \; r>1
    \end{cases} \nonumber\\
    P_r = 
    \begin{cases}
        \text{softmax}(\text{SubGatt}_{pool}(A_r,X_r)) \; , \; r=1 \\
        \text{softmax}(\text{GIN}_{r,pool}(A_r,X_r)) \; , \; 1 < r \leq R-1
    \end{cases}
\end{flalign}
Here, $Z_r \in \mathbb{R}^{N_r \times K}$ is the embedding matrix of the nodes of $G^r$. The softmax after the pooling is applied row-wise. $(i,j)$th element of $P_r \in \mathbb{R}^{N_r \times N_{r+1}}$ gives the probability of assigning node $v_i^r$ in $G^r$ to node $v_j^{r+1}$ in $G^{r+1}$. Based on these, the graph $G^{r+1}$ is constructed as follows,
\begin{flalign}\label{eq:adjFeat}
    A_{r+1} = P_r^T A_r P_r %\in \mathbb{R}^{N_{r+1} \times N_{r+1}} 
    \;;\;%\nonumber \\ %\;\; \text{and} \;\;
    X_{r+1} = P_r^T Z_r %\in \mathbb{R}^{N_{r+1} \times K}
\end{flalign}
The matrix $P_r$ contains information about how nodes in $G^r$ are mapped to the nodes of $G^{r+1}$, and the adjacency matrix $A_r$ contains information about the connection of nodes in $G^r$. Eq. \ref{eq:adjFeat} combines them to generate the connections between the nodes (i.e., the adjacency matrix $A_{r+1}$) of $G^{r+1}$. Node feature matrix $X_{r+1}$ of $G^{r+1}$ is also generated similarly.
As the embedding and pooling GNNs, we use SubGatt networks (Section \ref{sec:SubGatt}) only after the level graph 1. This is because other level graphs $G^r$ ($r>1$) have more number of soft edges (i.e., with probabilistic edge weights) due to use of softmax at the end of pooling layers. Hence, the number of neighboring rooted subtrees will be high in those level graphs and the chance of having discrete patterns will be less. We use GIN \cite{xu2018how} as the embedding and pooling GNNs for $G^r$, $r>1$. GIN has been shown to be the most powerful 1st order GNN and the $l$th layer of GIN can be defined as:
\begin{equation}\label{eq:GIN}
    h_v^{l+1} = MLP^l \Big( (1+\epsilon^l) h_v^l + \sum\limits_{u \in \mathcal{N}(v)} h_u^l \Big) 
\end{equation}
Here, $h_v^{l+1} \in \mathbb{R}^K$ is the hidden representation of the node $v$ in $l+1$th layer of GIN and $\epsilon$ is a learnable parameter.

\textbf{Intra-level attention layer}: As observed in \cite{lee2019self}, hierarchical GNNs often suffer because of the loss of information in various embedding and pooling layers, from the input graph to the last level graph summarizing the entire graph. Moreover, the learned hierarchy is often not perfect due to noisy structure of the real world graphs. To alleviate these problems, we propose to use attention mechanisms again, to combine features from different level graphs of our hierarchical architecture. We consider level graphs $G^2$ to $G^R$ for this, as their respective numbers of nodes are same across all the graphs in a dataset. We introduce \textit{intra-level attention layer} to obtain a global feature for each level graphs $G^r$, $\forall r=2,\cdots,R$.
%(since, $G^r$ has only one node, so no feature aggregation is required). 
More precisely, we use the convolution based self attention within the level graph $G^r$ as:
\begin{equation}\label{eq:intraLevel}
    e_r = \text{softmax}(\widetilde{D}_r^{-\frac{1}{2}}\widetilde{A_r}\widetilde{D}_r^{-\frac{1}{2}}X_r \theta) \; \in \mathbb{R}^{N_r} \;\;\text{and}\;\;
    x^r = X_r^T e_r \; \in \mathbb{R}^K 
\end{equation}
Here, the softmax to compute $e_r$ is taken so that a component of $e_r$ becomes the normalized (i.e., probabilistic) importance of the corresponding node in $G^r$. $\widetilde{A_r} = A_r + I_{N_r}$ is the adjacency matrix with added self loops of $G^r$. $\widetilde{D}$ is the diagonal matrix of dimension $N_r \times N_r$ with $\widetilde{D}(i,i) = \sum\limits_{j=1}^{N_r} \widetilde{A}_{ij}$. $\theta \in \mathbb{R}^K$ is the trainable vector of parameters of intra-level attention, which is shared across all the level graphs $G^r$, $\forall r=2,\cdots,R$. Intuitively, $\theta$ contains the importance of individual attributes and the components of $N_r$ dimensional $X_r \theta$ gives the same for each node. Finally, multiplying that with $\widetilde{D}_r^{-\frac{1}{2}}\widetilde{A_r}\widetilde{D}_r^{-\frac{1}{2}}$ produces the (normalized) importance of a node based on its own features and the features of immediate neighbors (for one layer of intra-level attention). Hence, $x^r$, which is a $K$ dimensional representation of the level graph $G^r$, is a sum of the features of the nodes weighted by the respective normalized node importance. Please note, the impact from the first few level graphs becomes noisy due to too many subsequent operations in a hierarchical pooling method. But representing level graphs separately by the proposed intra-level attention makes their impact more prominent.

\textbf{Inter-level attention layer}: 
%After the application of intra-level attention layer, we have one vector representation $x_r \in \mathbb{R}^K$ from $G^r$, $\forall$ $2 \leq r \leq R$. %and the feature vector $Z_R \in \mathbb{R}^K$ (from Eq. \ref{eq:embPool}). Let's set $x_R = Z_R$. 
This layer aims to get the final representation, referred as $x^G \in \mathbb{R}^K$, of the input graph from $x_2,\cdots,x_R$; as obtained from the intra-level attention layers.
%in contrast to a hierarchical pooling which just considers $Z_R$ to be the final representation. 
%We introduce \textit{inter-level attention layer}, which takes $x_2,\cdots,x_R$ as input and generates $x^G \in \mathbb{R}^K$ as the final graph representation to be fed to a neural classifier. We again use a self-attention as follows:
It is fed to a neural classifier. As different level graphs of the hierarchical representation have different importance to determine the label of the input graph, we propose to use the following self-attention mechanism.
\begin{equation}\label{eq:interLevel}
    \widetilde{e} = \text{softmax}(X_{inter} \widetilde{\theta}) \in \mathbb{R}^{R-1} \;\; \text{and}\;\;
    x^G = X_{inter}^T \widetilde{e} \in \mathbb{R}^K
\end{equation}
$X_{inter}$ is the $R-1 \times K$ dimensional matrix whose rows correspond to $x^r$ (the output of intra-level attention layer for $G^r$), $r=2,\cdots,R$. $\widetilde{\theta} \in \mathbb{R}^K$ is a trainable self attention vector.
Similar to Eq. \ref{eq:intraLevel}, softmax is taken to convert $\widetilde{e}$ to a probability distribution of importance of different graph levels. Finally, the vector representation $x^G$ of the input graph is computed as a weighted sum of representations of different level graphs $G^2,\cdots,G^R$.
$x^G$ is fed to a classification layer of the GNN, which is a dense layer followed by a softmax to classify the entire input graph in an end-to-end fashion. This completes the construction of SubGattPool architecture. 

\subsection{Key Insights of SubGattPool}\label{sec:keyInsights}
First layer of SubGattPool consists of an embedding SubGatt network and a pooling SubGatt network, which have a total of $O(KTD)$ trainable parameters. Consequent layers of SubGattPool have GIN as embedding and pooling layers, which have a total of $O(RKD)$ parameters. Total number of parameters for $R-2$ intra-level attention layers is $O(K)$, as $\theta \in \mathbb{R}^K$ is shared across the level graphs. Finally the inter-level attention layer has $O(K)$ parameters. Hence, total number of parameters to train in SubGattPool network is $O(KTD + RKD)$, which is independent of both the average number of nodes and the number of graphs in the dataset. We use ADAM (with learning rate set to 0.001) on the cross-entropy loss of graph classification to train these parameters.

%The hierarchical structure of SubGattPool (first row of Figure \ref{fig:SubGattPool}) is similar to that in DIFFPOOL \cite{ying2018hierarchical}. But, 
Please note that in contrast to existing hierarchical pooling mechanisms in GNN \cite{ying2018hierarchical,morris2019weisfeiler}, SubGattPool does not only rely on the last level of the GNN hierarchy to obtain the final graph representation. SubGattPool even may have more than 1 node in the last level graph. Essentially information from all the level graphs are aggregated through the attention layers.
SubGattPool is also less prone to information loss in the hierarchy and able to learn importance of individual nodes in a hierarchy (i.e., within a level graph) and the importance of different hierarchies. In terms of design, most of the existing GNNs use GCN embedding and pooling layers \cite{ying2018hierarchical}. Whereas, we propose subgraph attention mechanism through SubGatt network (discussed in Section \ref{sec:SubGatt}) and use it along with GIN as different embedding and pooling layers of SubGattPool. Following lemma shows that SubGattPool, though have different types of components in the overall architecture, satisfies a fundamental property required to be a graph neural network.

\begin{lemma}
	\textbf{For a graph $G=(V,E)$, with adjacency matrix $A \in \mathbb{R}^{|V| \times |V|}$ and node attribute matrix $X \in \mathbb{R}^{|V| \times D}$, let us use \\ $SubGattPool(A,X)$ to denote the final graph representation generated by SubGattPool on that graph. Let, $P \in \{0,1\}^{|V| \times |V|}$ is any permutation matrix. Assuming that the initialization and random selection strategies of the neural architecture are always the same, $SubGattPool(P A P^T, PX) = SubGattPool(A,X)$.}
\end{lemma}
\begin{proof}
	Please note that $P A P^T$ is the new adjacency matrix and $PX$ is the new feature matrix of the same graph $G$ under the node permutation defined by the permutation matrix $P$.
	So, to prove the above, we need to show that each component of SubGattPool is invariant to any node permutation. 
	First, SubGatt uses attention mechanism over the neighboring subgraphs through Equation \ref{eq:nodeEmb}. Clearly, different ordering of neighbors would not affect the node embeddings as we use a weighted sum aggregator where weights are learned through the subgraph attention.
	Next, the GIN aggregator (as in Equation \ref{eq:GIN}) is also invariant to node permutation. Thus, all the embedding and pooling layers (as shown in Figure \ref{fig:SubGattPool}) present in SubGattPool are invariant to different node permutations. Finally, both intra-level and inter-level attention mechanisms also do not depend on the ordering of the nodes in any level graph, as each of them uses sum aggregation with self-attention. Hence, SubGattPool is invariant to node permutations of the input graph.
\end{proof}

\section{Experimental Evaluation}\label{sec:exp}
This section describes the details of the experiments conducted on both real-life and synthetic datasets.

\begin{table*}
\centering
%\resizebox{\columnwidth}{!}{%
\begin{tabular}{*6c}
	\toprule
	\sffamily{Dataset} & \#Graphs & \#Max Nodes & Avg. Number of Nodes & \#Labels & \#Attributes\\
%    \sffamily{} & & & & Words & Distribution & links \\
    \hline
	\midrule
	\sffamily{MUTAG} & 188 & 28 & 17.93 & 2 & NA\\
	\sffamily{PTC} & 344 & 64 & 14.29 & 2 & NA\\
    \sffamily{PROTEINS} & 1113 & 620 & 	39.06 & 2 & 29\\
    \sffamily{NCI1} & 4110 & 111 & 29.87 & 2 & NA\\
    \sffamily{NCI109} & 4127 & 111 & 29.68 & 2 & NA\\
    \sffamily{IMDB-BINARY} & 1000 & 136 & 19.77 & 2 & NA\\
    \sffamily{IMDB-MULTI} & 1500 & 89 & 13.00 & 3 & NA\\
    
    \bottomrule
	\end{tabular}
%}
\caption{Statistics of different datasets used in our experiments}
\label{tab:datasets}
\end{table*}

\subsection{Experimental Setup for Graph Classification}\label{sec:dsBase}
We use 5 bioinformatics graph datasets (MUTAG, PTC, PROTEINS, NCI1 and NCI09) and 2 social network datasets (IMDB-BINARY and IMDB-MULTI) to evaluate the performance for graph classification. The details of these datasets can be found at (\url{https://bit.ly/39T079X}). Table \ref{tab:datasets} contains a high-level summary of these datasets. 
%We also conduct experiment on synthetic dataset in Section \ref{sec:synthetic} to get more insights.

To compare the performance of SubGattPool, we choose twenty state-of-the-art baseline algorithms from the domains of graph kernels, unsupervised graph representation and graph neural networks (Table \ref{tab:graphClassi}).
The reported accuracy numbers of the baseline algorithms are collected from \cite{maron2019provably,Sun2020InfoGraph,narayanan:mlg2017} where the same experimental setup is adopted. Thus, we avoid any degradation of the performance of the baseline algorithms due to insufficient parameter tuning and validation.

We adopt the same experimental setup as there in \cite{xu2018how}. We perform 10-fold cross validation and report the averaged accuracy and corresponding standard deviation for graph classification.
We keep the values of the hyperparameters to be the same across all the datasets, based on the averaged validation accuracy.
We set the pooling ratio (defined as $\gamma = \frac{N_{r+1}}{N_r}$, $\forall r \leq R-1$) at 0.5, the number of levels R=3 and the maximum subgraph size (T) to be 3. We sample L=12 subgraphs for each node in each epoch of SubGatt. Following most of the literature, we set the embedding dimension K to be 128. We use L2 normalization and dropout in SubGattPool architecture to make the training stable.

\begin{table*}
\centering
\resizebox{0.85\linewidth}{!}{%
\begin{tabular}{*8c} %& node2vec & GCN & GAT & DGI & \textbf{SubGatt}
	\toprule
	\sffamily{Algorithms} & \footnotesize MUTAG&PTC&PROTEINS&NCI1&NCI109&IMDB-B&IMDB-M \\
%    \sffamily{} & & & & Words & Distribution & links \\
    \hline
	\midrule
    
    GK \cite{shervashidze2009efficient}  &81.39$\pm$1.7&55.65$\pm$0.5&71.39$\pm$0.3&62.49$\pm$0.3&62.35$\pm$0.3&NA&NA \\
    RW \cite{vishwanathan2010graph}  &79.17$\pm$2.1&55.91$\pm$0.3&59.57$\pm$0.1&NA& NA&NA&NA\\
    PK \cite{neumann2016propagation} &76$\pm$2.7&59.5$\pm$2.4&73.68$\pm$0.7&82.54$\pm$0.5&NA&NA&NA\\
    WL \cite{shervashidze2011weisfeiler} &84.11$\pm$1.9&57.97$\pm$2.5&74.68$\pm$0.5&\textbf{84.46$\pm$0.5}&\textbf{85.12$\pm$0.3}&NA&NA\\
%    FGSD&\textbf{92.12}& 62.80& 73.42 &79.80 &78.84\\
    AWE-DD \cite{ivanov2018anonymous}  &NA& NA& NA& NA& NA&74.45$\pm$5.8&51.54$\pm$3.6\\
    AWE-FB \cite{ivanov2018anonymous} &87.87$\pm$9.7& NA &NA& NA& NA&73.13$\pm$3.2&51.58$\pm$4.6\\
    \hline
    node2vec \cite{grover2016node2vec} &72.63$\pm$10.20&58.85$\pm$8.00&57.49$\pm$3.57&54.89$\pm$1.61&52.68$\pm$1.56&NA&NA\\
    sub2vec \cite{adhikari2017distributed} &61.05$\pm$15.79&59.99$\pm$6.38&53.03$\pm$5.55&52.84$\pm$1.47&50.67$\pm$1.50&55.26$\pm$1.54&36.67$\pm$0.83\\
    graph2vec \cite{narayanan2017graph2vec} &83.15$\pm$9.25&60.17$\pm$6.86&73.30$\pm$2.05&73.22$\pm$1.81&74.26$\pm$1.47&71.1$\pm$0.54&50.44$\pm$0.87\\
    InfoGraph \cite{Sun2020InfoGraph} &89.01$\pm$1.13&61.65$\pm$1.43&NA&NA&NA&73.03$\pm$0.87&49.69$\pm$0.53\\
    \hline
    DGCNN \cite{zhang2018end} & 85.83$\pm$1.7& 58.59$\pm$2.5& 75.54$\pm$0.9& 74.44$\pm$0.5& NA&70.03$\pm$0.9&47.83$\pm$0.9\\
    PSCN \cite{niepert2016learning} &88.95$\pm$4.4&62.29$\pm$5.7&75$\pm$2.5&76.34$\pm$1.7&NA&71$\pm$2.3&45.23$\pm$2.8\\
    DCNN \cite{atwood2016diffusion} &NA&NA&61.29$\pm$1.6&56.61$\pm$1.0&NA&49.06$\pm$1.4&33.49$\pm$1.4\\
    ECC \cite{simonovsky2017dynamic} &76.11&NA&NA&76.82&75.03&NA&NA\\
    DGK \cite{yanardag2015deep} &87.44$\pm$2.7&60.08$\pm$2.6&75.68$\pm$0.5&80.31$\pm$0.5&80.32$\pm$0.3&66.96$\pm$0.6&44.55$\pm$0.5\\
    DIFFPOOL \cite{ying2018hierarchical} &83.56&NA&76.25&NA&NA&NA&47.91\\
  %  CCN&91.64$\pm$7.2&70.62$\pm$7.0&NA&76.27$\pm$4.1&75.54$\pm$3.4\\
    IGN \cite{maron2018invariant} &83.89$\pm$12.95&58.53$\pm$6.86&76.58$\pm$5.49&74.33$\pm$2.71&72.82$\pm$1.45&72.0$\pm$5.54&48.73$\pm$3.41\\
    GIN \cite{xu2018how} &89.4$\pm$5.6&64.6$\pm$7.0&76.2$\pm$2.8&82.7$\pm$1.7&NA&75.1$\pm$5.1&52.3$\pm$2.8\\
    1-2-3GNN \cite{morris2019weisfeiler} &86.1$\pm$&60.9$\pm$&75.5$\pm$&76.2$\pm$&NA&74.2$\pm$&49.5$\pm$\\
    3WL-GNN \cite{maron2019provably} &90.55$\pm$8.7&66.17$\pm$6.54&\textbf{77.2$\pm$4.73}&83.19$\pm$1.11&81.84$\pm$1.85&72.6$\pm$4.9&50$\pm$3.15\\
    \hline
    \textbf{SubGattPool}&\textbf{93.29$\pm$4.78}&\textbf{67.13$\pm$6.45}&76.92$\pm$3.44&82.59$\pm$1.42&80.95$\pm$1.76&\textbf{76.49$\pm$2.94}&\textbf{52.46$\pm$3.48}\\
    Rank & 1 & 1 & 2 & 3 & 3 & 1 & 1 \\
    \bottomrule
\end{tabular}
}
\caption{Classification accuracy (\%) of different algorithms (21 in total) for graph classification. NA denotes the case when the result of a baseline algorithm could not be found on that particular dataset from the existing literature. The last row `Rank' is the rank (1 being the highest position) of our proposed algorithm SubGattPool among all the algorithms present in the table.}
\label{tab:graphClassi}
\end{table*}

\subsection{Performance on Graph Classification}
Table \ref{tab:graphClassi} shows the performance of SubGattPool along with the diverse set of baseline algorithms for graph classification on multiple real-world datasets. From the results, we can observe that SubGattPool is able to improve the state-of-the-art on MUTAG, PTC, IMDB-B and IMDB-M for graph classification. On PROTEINS, the performance gap with the best performing baseline (which is 3WL-GNN \cite{maron2019provably} for both) is less than 1\%. But on NCI1 and NCI109, WL kernel turns out to be the best performing algorithm with a good margin ($>1\%$) from all the GNN based algorithms.
It is interesting to note that SubGattPool is able to outperform existing hierarchical GNN algorithms DIFFPOOL and 1-2-3GNN consistently on all the datasets. This is because of the use of (i) attention over subgraphs in SubGatt embedding and pooling layers, and (ii) use of intra-level and inter-level attention mechanisms over different level graphs which makes the overall architecture more robust and reduces information loss.
In terms of standard deviation, SubGattPool is highly competitive and often better than most of the better performing GNNs (specially GIN and 3WL-GNN).

\begin{figure}[h!]
  \centering
  \begin{subfigure}{0.33\linewidth}
  \centering
    \includegraphics[width=\linewidth]{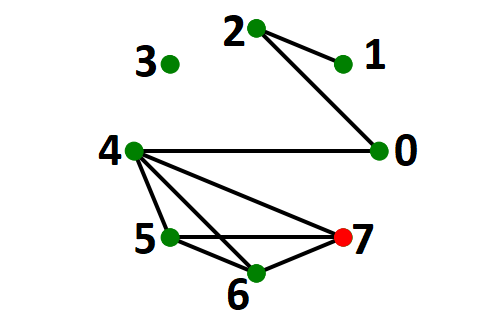}
    \caption{}
    \label{fig:simGraph}
  \end{subfigure}
  \begin{subfigure}{0.65\linewidth}
  \centering
    \includegraphics[width=0.9\linewidth]{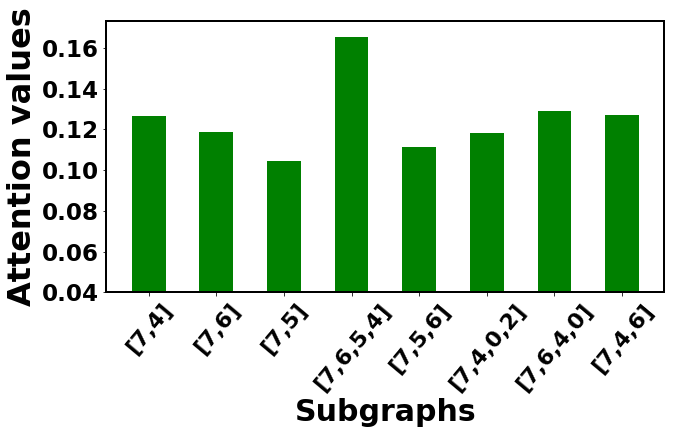}
    \caption{}
    \label{fig:subbAtten}
  \end{subfigure}
  \caption{ (a) A sample graph. (b) Normalized attention values of different subgraphs selected for the node 7 of the Graph in (a). Clearly, attention to the clique of size 4 is more than all the other subgraphs.}
  \label{fig:attention}
\end{figure}

\subsection{Interpretation of Subgraph Attention via Synthetic Experiment}\label{sec:synthetic}
Subgraph attention is a key component of SubGattPool. Here, we validate the learned attention values on different subgraphs by conducting an experiment on a small synthetic dataset containing 50 graphs, and each graph having 8 nodes. Each graph has 2 balanced communities and exactly for 50\% of the graphs, one community consists of a clique of size 4. We label a graph with +1 if the clique of size 4 is present, otherwise the label is -1. The goal of this experiment is to see if SubGattPool is able to learn this simple rule of graph classification by paying proper attention to the substructure which determines the label of a graph.

We run SubGattPool on this synthetic dataset, with $K=16$, $T=4$, $L=12$ \#SubGatt layers=1, $\gamma=0.75$ and $R=3$. Once the training is complete, we randomly select a graph from the positive class and a node in it and plot the attention values of all the subgraphs selected in the last epoch for that node, in Figure \ref{fig:attention}. Clearly, the attention value corresponding to the clique (containing the nodes 7, 6, 5 and 4) is much higher than that to the other subgraphs. We have manually verified the same observation on multiple graphs in this synthetic dataset. Thus, SubGattPool is able to pay more attention to the correct substructure (i.e., subgraph) and pay less attention to other irrelevant substructures. This also explains the robust behavior of SubGattPool.
%Hence, our algorithm is able to identify and pay more importance to the structure which determines the label of the graph.

\begin{figure*}[h!]
  \centering
  \begin{subfigure}[b]{0.24\linewidth}
    \includegraphics[width=\linewidth]{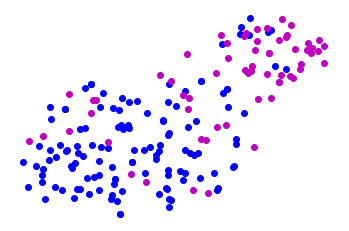}
    \caption{\scriptsize DIFFPOOL}
    \label{fig:vizDiffpool}
  \end{subfigure}
  \begin{subfigure}[b]{0.24\linewidth}
    \includegraphics[width=\linewidth]{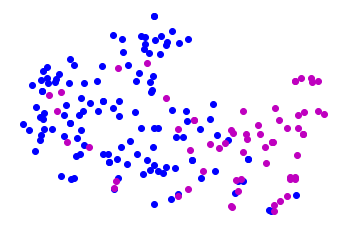}
    \caption{\scriptsize SubGattPool \textbackslash SubGatt}
    \label{fig:viz-subgatt}
  \end{subfigure}
  \begin{subfigure}[b]{0.24\linewidth}
    \includegraphics[width=\linewidth]{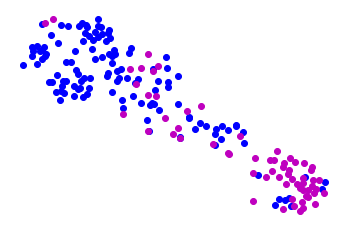}
    \caption{\scriptsize SubGattPool \textbackslash I-I-L-A}
    \label{fig:viz-IIL}
  \end{subfigure}
  \begin{subfigure}[b]{0.24\linewidth}
    \includegraphics[width=\linewidth]{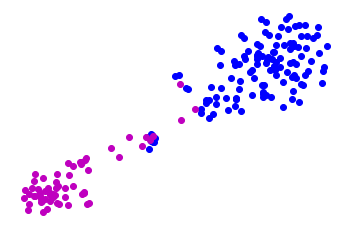}
    \caption{\scriptsize SubGattPool}
    \label{fig:vizSubGattPool}
  \end{subfigure}
  \caption{t-SNE visualization of the graphs from MUTAG (different colors show different labels of the graphs) by the representations generated by: (a) DIFFPOOL; (b) SubGattPool, but the SubGatt embedding and pooling layers being replaced by GIN; (c) SubGattPool without intra and inter layer attention; (d) the complete SubGattPool network. Compared to (a), there is improvement of performances for both the SubGatt layer and the intra/inter-level attention individually. Finally different classes are separated most by SubGattPool which again shows the merit of the proposed algorithm.}
  \label{fig:visMUTAG}
\end{figure*}

\begin{figure*}[h!]
  \centering
  \begin{subfigure}[b]{0.24\linewidth}
    \includegraphics[width=\linewidth]{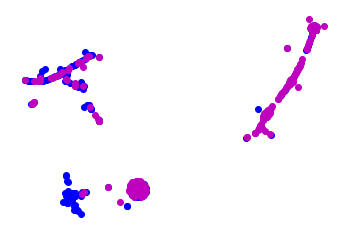}
    \caption{\scriptsize DIFFPOOL}
    \label{fig:vizPTCDiffpool}
  \end{subfigure}
  \begin{subfigure}[b]{0.24\linewidth}
    \includegraphics[width=\linewidth]{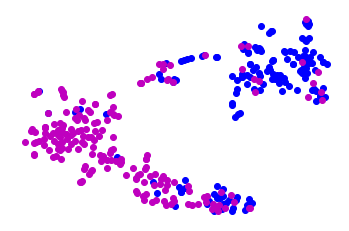}
    \caption{\scriptsize SubGattPool \textbackslash SubGatt}
    \label{fig:vizPTC-subgatt}
  \end{subfigure}
  \begin{subfigure}[b]{0.24\linewidth}
    \includegraphics[width=\linewidth]{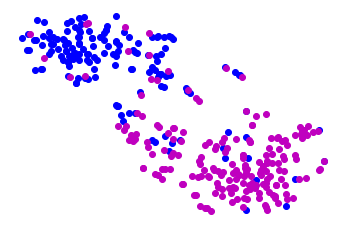}
    \caption{\scriptsize SubGattPool \textbackslash I-I-L-A}
    \label{fig:vizPTC-IIL}
  \end{subfigure}
  \begin{subfigure}[b]{0.24\linewidth}
    \includegraphics[width=\linewidth]{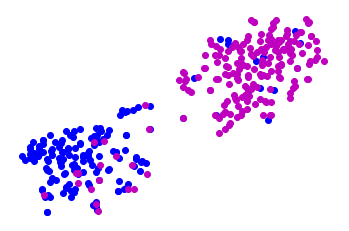}
    \caption{\scriptsize SubGattPool}
    \label{fig:vizPTC_SubGattPool}
  \end{subfigure}
  \caption{t-SNE visualization of the graphs from PTC (different colors show different labels of the graphs) by the representations generated by different GNN algrotihms. The description of each variant of SubGattPool is exactly the same as in Figure \ref{fig:visMUTAG}. Again for PTC also, we can see the performances of different variants of SubGattPool are better than that of DIFFPOOL and the overall performance of SubGattPool in visualizing PTC is better than the other variants which are obtained by removing one or more components from SubGattPool.}
  \label{fig:visPTC}
\end{figure*}

\subsection{Graph Clustering}\label{sec:graphClus}
Though our proposed algorithm SubGattPool is for graph classification, we also wants to check the quality of the graph representations $x^G$, $\forall G \in \mathcal{G}$ obtained in SubGattPool through graph clustering. We use only a subset of recently proposed GNN based algorithm as baselines in this experiment. We use similar hyperparameter values (discussed in Section \ref{sec:exp}) as applicable and adopt same hyperparameter tuning strategy to obtain the graph representation for all the algorithms considered. The vector representations obtained for all the graphs by a GNN are given to K-Means++ \cite{arthur2006k} algorithm to get the clusters. 
To evaluate the quality of clustering, we use unsupervised clustering accuracy \cite{bandyopadhyay2019outlier,bandyopadhyay2020outlier} which uses different permutations of the labels and chooses the label ordering which gives the best possible accuracy $Acc(\mathcal{\hat{C}},\mathcal{C}) = \max_{\mathcal{P}} \frac{ \sum\limits_{i=1}^N \mathbf{1}(\mathcal{P}(\mathcal{\hat{C}}_i)  = \mathcal{C}_i) }{N}$.
% \begin{align}
% \end{align}
Here $\mathcal{C}$ is the ground truth labeling of the dataset such that $\mathcal{C}_i$ gives the ground truth label of $i$th data point. Similarly $\mathcal{\hat{C}}$ is the clustering assignments discovered by some algorithm, and $\mathcal{P}$ is a permutation of the set of labels.
We assume $\mathbf{1}$ to be a logical operator which returns 1 when the argument is true, otherwise returns 0. Table \ref{tab:graphClus} shows that SubGattPool is able to outperform all the baselines we used for graph clustering on all the three datasets. Please note that DGI and InfoGraph derive the graph embeddings in an unsupervised way, whereas DIFFPOOL and SubGattPool use supervision. Naturally, the performance of the later two are better on all the datasets. Further, the use of subgraph attention along with the hierarchical attention layers helps SubGattPool to perform consistently better than DIFFPOOL which is also hierarchical in nature.

\begin{table}
\centering
%\small
\begin{tabular}{*4c} %& node2vec & GCN & GAT & DGI & \textbf{SubGatt}
	\toprule
	\sffamily{Algorithms} & \footnotesize MUTAG&PROTEINS&IMDB-M \\
%    \sffamily{} & & & & Words & Distribution & links \\
    \hline
	\midrule
    DGI&74.73&59.20&36.83\\
    InfoGraph &77.65&59.93&35.93\\
    %\hline
    DIFFPOOL&82.08&60.81&41.72\\
    SubGattPool&90.68&65.45&50.23\\
    \bottomrule
\end{tabular}
\caption{Clustering accuracy(\%).}
\label{tab:graphClus}
\end{table}

\begin{figure*}[h!]
  \centering
  \begin{subfigure}[b]{0.24\linewidth}
    \includegraphics[width=\linewidth]{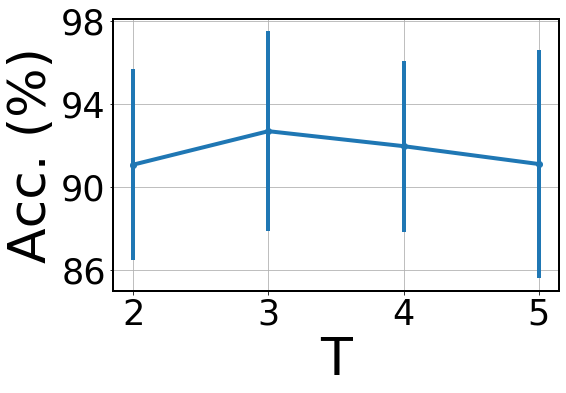}
    \caption{}
    \label{fig:subgraphSize}
  \end{subfigure}
  \begin{subfigure}[b]{0.24\linewidth}
    \includegraphics[width=\linewidth]{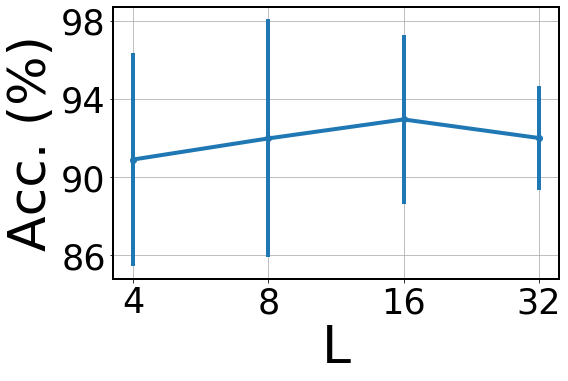}
    \caption{}
    \label{fig:NumSampled}
  \end{subfigure}
  \begin{subfigure}[b]{0.24\linewidth}
    \includegraphics[width=\linewidth]{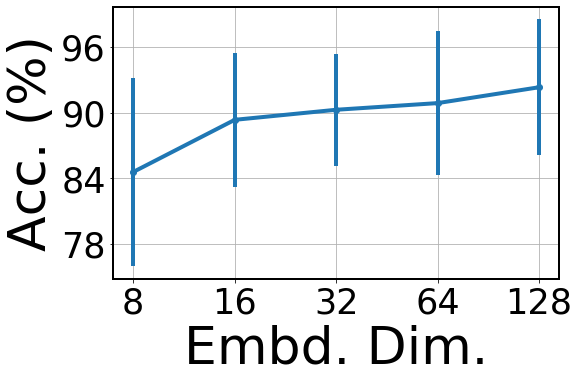}
    \caption{}
    \label{fig:EmbDim}
  \end{subfigure}
  \begin{subfigure}[b]{0.24\linewidth}
    \includegraphics[width=\linewidth]{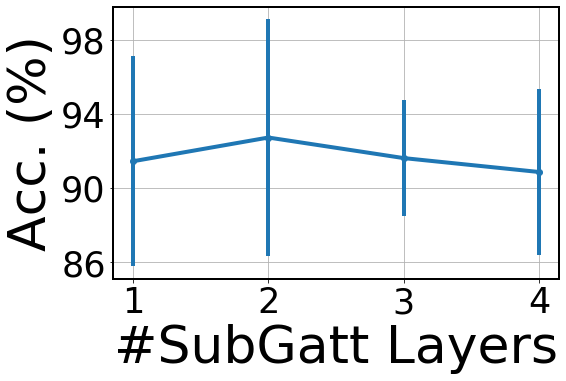}
    \caption{}
    \label{fig:SubGattLayers}
  \end{subfigure}
  \caption{Sensitivity analysis of SubGattPool for graph classification on MUTAG with respect to different hyper-parameters: (a) Maximum subgraph size, (b) Number of subgraphs sampled per epoch for each node, (c) Embedding dimension and (d) Number of SubGatt layers in SubGattPool. }
  \label{fig:sensitySubgatt}
\end{figure*} 

\subsection{Model Ablation Study}\label{sec:ablation}
SubGattPool has mainly two novel components. They are the SubGatt layer, and the intra-level and inter-level attention layers which makes SubGattPool a mixture of both global and hierarchical pooling strategy. 
To see the usefulness of each component, we show the performance after removing that component from SubGattPool. 
We chose graph visualization of MUTAG in Figure \ref{fig:visMUTAG} and graph visualization of PTC in Figure \ref{fig:visPTC} as the downstream tasks for this experiment.  We use t-SNE \cite{vanDerMaaten2008} to convert the graph embeddings into two dimensional plane. Different colors represent different labels of the graphs and the performance is better when different colors form different clusters in the plot.
We choose DIFFPOOL as the base model in Figure \ref{fig:vizDiffpool} because it is also a hierarchical graph representation technique.
In Figure \ref{fig:viz-subgatt}, we replace the SubGatt embedding and pooling layers by GIN embedding and pooling layers in SubGattPool (refer Figure \ref{fig:SubGattPool}). Similarly, in Figure \ref{fig:viz-IIL}, we remove inter and intra layer attention and obtain the graph representation from the last level graph (by creating only one node there) in SubGattPool. Finally, Figure \ref{fig:vizSubGattPool} shows the performance by SubGattPool, which combines all these components into a single network. Clearly, the performances in Figure \ref{fig:viz-subgatt} and \ref{fig:viz-IIL} are better than that in Figure \ref{fig:vizDiffpool}, but the best performance is achieved in Figure \ref{fig:vizSubGattPool} which uses the complete SubGattPool network on MUTAG. The same observation of improved performance of the variants of SubGattPool over DIFFPOOL and the performance of SubGattPool over its variants is also prominent in Figure \ref{fig:visPTC} on PTC dataset. This clearly shows the individual and combined usefulness of various components of SubGattPool for graph representation.

\subsection{Sensitivity Analysis}\label{sec:sensitivity}
%We also conduct sensitivity analysis of SubGattPool w.r.t. all the important hyperparameters, as explained in Figure \ref{fig:sensitySubgatt}. It turns out that SubGattPool is reasonably stable with respect to them.
%\subsection{Sensitivity Analysis of SubGatt Network}
We aim to conduct sensitivity analysis of the proposed algorithm in this section. SubGatt network has four important hyperparameters. They are: (i) Maximum size of a subgraph ($T$), (ii) Number of subgraphs sampled per node in each epoch ($L$) and (iii) Dimension of the final node representation or embedding ($K$) (See Eq. \ref{eq:nodeEmb}) and (iv) Number of SubGatt layers used in the network. We conduct graph classification experiment on MUTAG to see the sensitivity of SubGattPool with respect to each of these hyperparameters. Figure \ref{fig:sensitySubgatt} shows the variation of the performance of SubGattPool network for graph classification with respect to all these hyper-parameters. We have shown both average graph classification accuracy and standard deviation over 10 repetitions for each experiment.

From Figure \label{fig:subgraphSize}, we can see that the performance of SubGattPool on MUTAG improves when maximum length of subgraph is set to 3. As the average size of a graph in MUTAG is quite small, a subgraph of size more than 3 does not help. Similarly, Figure \label{fig:NumSampled} shows that with increasing number of samples ($L$) for each node in an epoch, the performance of SubGattPool improves first, and then saturates. The same observation can be made in Figure \label{fig:EmbDim} for embedding dimension of the graphs. We use SubGatt as the embedding and pooling layers of the GNN after level graph 1. Figure \label{fig:SubGattLayers} shows that best performance on MUTAG is obtained with 2 layers of SubGatt. Adding more number of layers actually deteriorates the performance because of oversmoothing which is a well-known problem of graph neural networks \cite{luan2019break}.
Overall, the variation is as expected and often less with respect to each hyper-parameter and hence it shows the robustness of SubGattPool. Please note, when we are varying one hyper-parameter of SubGattPool, the values of all other hyper-parameters are fixed to the values mentioned in Section \ref{sec:dsBase}.

\section{Conclusion}\label{sec:dis}
We have proposed a novel GNN based robust graph classification algorithm called SubGattPool which uses higher order attention over the subgraphs of a graph and also addresses some shortcomings of the existing hierarchical graph representation techniques.
We conduct experiments with both real world and synthetic graph datasets on multiple graph-level downstream tasks to show the robustness of our algorithm.
We are also able to improve the state-of-the-art graph classification performance on four popularly used graph datasets. In future, we would like to theoretically examine the expressiveness power of SubGatt and SubGattPool for node and graph representations respectively. We will also analyze and see the recovery of communities in a graph in the hierarchical structure of SubGattPool. Overall, we believe that this work would encourage further development in the area of hierarchical graph representation and classification.

\balance

\bibliographystyle{ACM-Reference-Format}
\bibliography{sample-base}


\begin{thebibliography}{47}

%%% the Web address.

\ifx \showCODEN    \undefined \def \showCODEN     #1{\unskip}     \fi
\ifx \showDOI      \undefined \def \showDOI       #1{#1}\fi
\ifx \showISBNx    \undefined \def \showISBNx     #1{\unskip}     \fi
\ifx \showISBNxiii \undefined \def \showISBNxiii  #1{\unskip}     \fi
\ifx \showISSN     \undefined \def \showISSN      #1{\unskip}     \fi
\ifx \showLCCN     \undefined \def \showLCCN      #1{\unskip}     \fi
\ifx \shownote     \undefined \def \shownote      #1{#1}          \fi
\ifx \showarticletitle \undefined \def \showarticletitle #1{#1}   \fi
\ifx \showURL      \undefined \def \showURL       {\relax}        \fi
% invisible to TeX
\providecommand\bibfield[2]{#2}
\providecommand\bibinfo[2]{#2}
\providecommand\natexlab[1]{#1}
\providecommand\showeprint[2][]{arXiv:#2}

\bibitem[\protect\citeauthoryear{Abu-El-Haija, Perozzi, Kapoor, Alipourfard,
  Lerman, Harutyunyan, Ver~Steeg, and Galstyan}{Abu-El-Haija
  et~al\mbox{.}}{2019}]%
        {abu2019mixhop}
\bibfield{author}{\bibinfo{person}{Sami Abu-El-Haija}, \bibinfo{person}{Bryan
  Perozzi}, \bibinfo{person}{Amol Kapoor}, \bibinfo{person}{Nazanin
  Alipourfard}, \bibinfo{person}{Kristina Lerman}, \bibinfo{person}{Hrayr
  Harutyunyan}, \bibinfo{person}{Greg Ver~Steeg}, {and} \bibinfo{person}{Aram
  Galstyan}.} \bibinfo{year}{2019}\natexlab{}.
\newblock \showarticletitle{MixHop: Higher-Order Graph Convolutional
  Architectures via Sparsified Neighborhood Mixing}. In
  \bibinfo{booktitle}{\emph{International Conference on Machine Learning}}.
  \bibinfo{pages}{21--29}.
\newblock


\bibitem[\protect\citeauthoryear{Adhikari, Zhang, Ramakrishnan, and
  Prakash}{Adhikari et~al\mbox{.}}{2017}]%
        {adhikari2017distributed}
\bibfield{author}{\bibinfo{person}{Bijaya Adhikari}, \bibinfo{person}{Yao
  Zhang}, \bibinfo{person}{Naren Ramakrishnan}, {and} \bibinfo{person}{B~Aditya
  Prakash}.} \bibinfo{year}{2017}\natexlab{}.
\newblock \showarticletitle{Distributed representations of subgraphs}. In
  \bibinfo{booktitle}{\emph{2017 IEEE International Conference on Data Mining
  Workshops (ICDMW)}}. IEEE, \bibinfo{pages}{111--117}.
\newblock


\bibitem[\protect\citeauthoryear{Arthur and Vassilvitskii}{Arthur and
  Vassilvitskii}{2006}]%
        {arthur2006k}
\bibfield{author}{\bibinfo{person}{David Arthur} {and} \bibinfo{person}{Sergei
  Vassilvitskii}.} \bibinfo{year}{2006}\natexlab{}.
\newblock \bibinfo{booktitle}{\emph{k-means++: The advantages of careful
  seeding}}.
\newblock \bibinfo{type}{{T}echnical {R}eport}.
  \bibinfo{institution}{Stanford}.
\newblock


\bibitem[\protect\citeauthoryear{Atwood and Towsley}{Atwood and
  Towsley}{2016}]%
        {atwood2016diffusion}
\bibfield{author}{\bibinfo{person}{James Atwood} {and} \bibinfo{person}{Don
  Towsley}.} \bibinfo{year}{2016}\natexlab{}.
\newblock \showarticletitle{Diffusion-convolutional neural networks}. In
  \bibinfo{booktitle}{\emph{Advances in Neural Information Processing
  Systems}}. \bibinfo{pages}{1993--2001}.
\newblock


\bibitem[\protect\citeauthoryear{Bandyopadhyay, Kara, Kannan, and
  Murty}{Bandyopadhyay et~al\mbox{.}}{2018}]%
        {bandyopadhyay2018fscnmf}
\bibfield{author}{\bibinfo{person}{Sambaran Bandyopadhyay},
  \bibinfo{person}{Harsh Kara}, \bibinfo{person}{Aswin Kannan}, {and}
  \bibinfo{person}{M~Narasimha Murty}.} \bibinfo{year}{2018}\natexlab{}.
\newblock \showarticletitle{Fscnmf: Fusing structure and content via
  non-negative matrix factorization for embedding information networks}.
\newblock \bibinfo{journal}{\emph{arXiv preprint arXiv:1804.05313}}
  (\bibinfo{year}{2018}).
\newblock


\bibitem[\protect\citeauthoryear{Bandyopadhyay, Lokesh, and
  Murty}{Bandyopadhyay et~al\mbox{.}}{2019}]%
        {bandyopadhyay2019outlier}
\bibfield{author}{\bibinfo{person}{Sambaran Bandyopadhyay}, \bibinfo{person}{N
  Lokesh}, {and} \bibinfo{person}{M~Narasimha Murty}.}
  \bibinfo{year}{2019}\natexlab{}.
\newblock \showarticletitle{Outlier Aware Network Embedding for Attributed
  Networks}. In \bibinfo{booktitle}{\emph{Proceedings of the AAAI Conference on
  Artificial Intelligence}}, Vol.~\bibinfo{volume}{33}.
  \bibinfo{pages}{12--19}.
\newblock


\bibitem[\protect\citeauthoryear{Bandyopadhyay, Vivek, and Murty}{Bandyopadhyay
  et~al\mbox{.}}{2020}]%
        {bandyopadhyay2020outlier}
\bibfield{author}{\bibinfo{person}{Sambaran Bandyopadhyay},
  \bibinfo{person}{Saley~Vishal Vivek}, {and} \bibinfo{person}{MN Murty}.}
  \bibinfo{year}{2020}\natexlab{}.
\newblock \showarticletitle{Outlier Resistant Unsupervised Deep Architectures
  for Attributed Network Embedding}. In \bibinfo{booktitle}{\emph{Proceedings
  of the 13th International Conference on Web Search and Data Mining}}.
  \bibinfo{pages}{25--33}.
\newblock


\bibitem[\protect\citeauthoryear{Borgwardt and Kriegel}{Borgwardt and
  Kriegel}{2005}]%
        {borgwardt2005shortest}
\bibfield{author}{\bibinfo{person}{Karsten~M Borgwardt} {and}
  \bibinfo{person}{Hans-Peter Kriegel}.} \bibinfo{year}{2005}\natexlab{}.
\newblock \showarticletitle{Shortest-path kernels on graphs}. In
  \bibinfo{booktitle}{\emph{Fifth IEEE international conference on data mining
  (ICDM'05)}}. IEEE, \bibinfo{pages}{8--pp}.
\newblock


\bibitem[\protect\citeauthoryear{Chami, Ying, R{\'e}, and Leskovec}{Chami
  et~al\mbox{.}}{2019}]%
        {chami2019hyperbolic}
\bibfield{author}{\bibinfo{person}{Ines Chami}, \bibinfo{person}{Zhitao Ying},
  \bibinfo{person}{Christopher R{\'e}}, {and} \bibinfo{person}{Jure Leskovec}.}
  \bibinfo{year}{2019}\natexlab{}.
\newblock \showarticletitle{Hyperbolic graph convolutional neural networks}. In
  \bibinfo{booktitle}{\emph{Advances in Neural Information Processing
  Systems}}. \bibinfo{pages}{4869--4880}.
\newblock


\bibitem[\protect\citeauthoryear{Defferrard, Bresson, and
  Vandergheynst}{Defferrard et~al\mbox{.}}{2016}]%
        {defferrard2016convolutional}
\bibfield{author}{\bibinfo{person}{Micha{\"e}l Defferrard},
  \bibinfo{person}{Xavier Bresson}, {and} \bibinfo{person}{Pierre
  Vandergheynst}.} \bibinfo{year}{2016}\natexlab{}.
\newblock \showarticletitle{Convolutional neural networks on graphs with fast
  localized spectral filtering}. In \bibinfo{booktitle}{\emph{Advances in
  neural information processing systems}}. \bibinfo{pages}{3844--3852}.
\newblock


\bibitem[\protect\citeauthoryear{Duvenaud, Maclaurin, Iparraguirre, Bombarell,
  Hirzel, Aspuru-Guzik, and Adams}{Duvenaud et~al\mbox{.}}{2015}]%
        {duvenaud2015convolutional}
\bibfield{author}{\bibinfo{person}{David~K Duvenaud}, \bibinfo{person}{Dougal
  Maclaurin}, \bibinfo{person}{Jorge Iparraguirre}, \bibinfo{person}{Rafael
  Bombarell}, \bibinfo{person}{Timothy Hirzel}, \bibinfo{person}{Al{\'a}n
  Aspuru-Guzik}, {and} \bibinfo{person}{Ryan~P Adams}.}
  \bibinfo{year}{2015}\natexlab{}.
\newblock \showarticletitle{Convolutional networks on graphs for learning
  molecular fingerprints}. In \bibinfo{booktitle}{\emph{Advances in neural
  information processing systems}}. \bibinfo{pages}{2224--2232}.
\newblock


\bibitem[\protect\citeauthoryear{Gilmer, Schoenholz, Riley, Vinyals, and
  Dahl}{Gilmer et~al\mbox{.}}{2017}]%
        {gilmer2017neural}
\bibfield{author}{\bibinfo{person}{Justin Gilmer}, \bibinfo{person}{Samuel~S
  Schoenholz}, \bibinfo{person}{Patrick~F Riley}, \bibinfo{person}{Oriol
  Vinyals}, {and} \bibinfo{person}{George~E Dahl}.}
  \bibinfo{year}{2017}\natexlab{}.
\newblock \showarticletitle{Neural message passing for quantum chemistry}. In
  \bibinfo{booktitle}{\emph{Proceedings of the 34th International Conference on
  Machine Learning-Volume 70}}. JMLR, \bibinfo{pages}{1263--1272}.
\newblock


\bibitem[\protect\citeauthoryear{Grover and Leskovec}{Grover and
  Leskovec}{2016}]%
        {grover2016node2vec}
\bibfield{author}{\bibinfo{person}{Aditya Grover} {and} \bibinfo{person}{Jure
  Leskovec}.} \bibinfo{year}{2016}\natexlab{}.
\newblock \showarticletitle{node2vec: Scalable feature learning for networks}.
  In \bibinfo{booktitle}{\emph{Proceedings of the 22nd ACM SIGKDD international
  conference on Knowledge discovery and data mining}}. ACM,
  \bibinfo{pages}{855--864}.
\newblock


\bibitem[\protect\citeauthoryear{Hamilton, Ying, and Leskovec}{Hamilton
  et~al\mbox{.}}{2017}]%
        {hamilton2017inductive}
\bibfield{author}{\bibinfo{person}{Will Hamilton}, \bibinfo{person}{Zhitao
  Ying}, {and} \bibinfo{person}{Jure Leskovec}.}
  \bibinfo{year}{2017}\natexlab{}.
\newblock \showarticletitle{Inductive representation learning on large graphs}.
  In \bibinfo{booktitle}{\emph{Advances in Neural Information Processing
  Systems}}. \bibinfo{pages}{1025--1035}.
\newblock


\bibitem[\protect\citeauthoryear{Ivanov and Burnaev}{Ivanov and
  Burnaev}{2018}]%
        {ivanov2018anonymous}
\bibfield{author}{\bibinfo{person}{Sergey Ivanov} {and} \bibinfo{person}{Evgeny
  Burnaev}.} \bibinfo{year}{2018}\natexlab{}.
\newblock \showarticletitle{Anonymous walk embeddings}.
\newblock \bibinfo{journal}{\emph{arXiv preprint arXiv:1805.11921}}
  (\bibinfo{year}{2018}).
\newblock


\bibitem[\protect\citeauthoryear{Kashima, Tsuda, and Inokuchi}{Kashima
  et~al\mbox{.}}{2003}]%
        {kashima2003marginalized}
\bibfield{author}{\bibinfo{person}{Hisashi Kashima}, \bibinfo{person}{Koji
  Tsuda}, {and} \bibinfo{person}{Akihiro Inokuchi}.}
  \bibinfo{year}{2003}\natexlab{}.
\newblock \showarticletitle{Marginalized kernels between labeled graphs}. In
  \bibinfo{booktitle}{\emph{ICML}}. \bibinfo{pages}{321--328}.
\newblock


\bibitem[\protect\citeauthoryear{Kipf and Welling}{Kipf and Welling}{2016}]%
        {kipf2016semi}
\bibfield{author}{\bibinfo{person}{Thomas~N Kipf} {and} \bibinfo{person}{Max
  Welling}.} \bibinfo{year}{2016}\natexlab{}.
\newblock \showarticletitle{Semi-supervised classification with graph
  convolutional networks}.
\newblock \bibinfo{journal}{\emph{arXiv preprint arXiv:1609.02907}}
  (\bibinfo{year}{2016}).
\newblock


\bibitem[\protect\citeauthoryear{Knyazev, Taylor, and Amer}{Knyazev
  et~al\mbox{.}}{2019}]%
        {knyazev2019understanding}
\bibfield{author}{\bibinfo{person}{Boris Knyazev}, \bibinfo{person}{Graham~W
  Taylor}, {and} \bibinfo{person}{Mohamed~R Amer}.}
  \bibinfo{year}{2019}\natexlab{}.
\newblock \showarticletitle{Understanding attention in graph neural networks}.
  In \bibinfo{booktitle}{\emph{Advances in Neural Information Processing
  Systems}}.
\newblock


\bibitem[\protect\citeauthoryear{Lee, Lee, and Kang}{Lee et~al\mbox{.}}{2019}]%
        {lee2019self}
\bibfield{author}{\bibinfo{person}{Junhyun Lee}, \bibinfo{person}{Inyeop Lee},
  {and} \bibinfo{person}{Jaewoo Kang}.} \bibinfo{year}{2019}\natexlab{}.
\newblock \showarticletitle{Self-Attention Graph Pooling}. In
  \bibinfo{booktitle}{\emph{International Conference on Machine Learning}}.
  \bibinfo{pages}{3734--3743}.
\newblock


\bibitem[\protect\citeauthoryear{Lee, Rossi, and Kong}{Lee
  et~al\mbox{.}}{2018}]%
        {lee2018graph}
\bibfield{author}{\bibinfo{person}{John~Boaz Lee}, \bibinfo{person}{Ryan
  Rossi}, {and} \bibinfo{person}{Xiangnan Kong}.}
  \bibinfo{year}{2018}\natexlab{}.
\newblock \showarticletitle{Graph classification using structural attention}.
  In \bibinfo{booktitle}{\emph{Proceedings of the 24th ACM SIGKDD International
  Conference on Knowledge Discovery \& Data Mining}}. ACM,
  \bibinfo{pages}{1666--1674}.
\newblock


\bibitem[\protect\citeauthoryear{Luan, Zhao, Chang, and Precup}{Luan
  et~al\mbox{.}}{2019}]%
        {luan2019break}
\bibfield{author}{\bibinfo{person}{Sitao Luan}, \bibinfo{person}{Mingde Zhao},
  \bibinfo{person}{Xiao-Wen Chang}, {and} \bibinfo{person}{Doina Precup}.}
  \bibinfo{year}{2019}\natexlab{}.
\newblock \showarticletitle{Break the Ceiling: Stronger Multi-scale Deep Graph
  Convolutional Networks}. In \bibinfo{booktitle}{\emph{Advances in Neural
  Information Processing Systems}}. \bibinfo{pages}{10943--10953}.
\newblock


\bibitem[\protect\citeauthoryear{Maron, Ben-Hamu, Serviansky, and Lipman}{Maron
  et~al\mbox{.}}{2019}]%
        {maron2019provably}
\bibfield{author}{\bibinfo{person}{Haggai Maron}, \bibinfo{person}{Heli
  Ben-Hamu}, \bibinfo{person}{Hadar Serviansky}, {and} \bibinfo{person}{Yaron
  Lipman}.} \bibinfo{year}{2019}\natexlab{}.
\newblock \showarticletitle{Provably powerful graph networks}. In
  \bibinfo{booktitle}{\emph{Advances in Neural Information Processing
  Systems}}. \bibinfo{pages}{2153--2164}.
\newblock


\bibitem[\protect\citeauthoryear{Maron, Ben-Hamu, Shamir, and Lipman}{Maron
  et~al\mbox{.}}{2018}]%
        {maron2018invariant}
\bibfield{author}{\bibinfo{person}{Haggai Maron}, \bibinfo{person}{Heli
  Ben-Hamu}, \bibinfo{person}{Nadav Shamir}, {and} \bibinfo{person}{Yaron
  Lipman}.} \bibinfo{year}{2018}\natexlab{}.
\newblock \showarticletitle{Invariant and equivariant graph networks}.
\newblock \bibinfo{journal}{\emph{arXiv preprint arXiv:1812.09902}}
  (\bibinfo{year}{2018}).
\newblock


\bibitem[\protect\citeauthoryear{Morris, Ritzert, Fey, Hamilton, Lenssen,
  Rattan, and Grohe}{Morris et~al\mbox{.}}{2019}]%
        {morris2019weisfeiler}
\bibfield{author}{\bibinfo{person}{Christopher Morris}, \bibinfo{person}{Martin
  Ritzert}, \bibinfo{person}{Matthias Fey}, \bibinfo{person}{William~L
  Hamilton}, \bibinfo{person}{Jan~Eric Lenssen}, \bibinfo{person}{Gaurav
  Rattan}, {and} \bibinfo{person}{Martin Grohe}.}
  \bibinfo{year}{2019}\natexlab{}.
\newblock \showarticletitle{Weisfeiler and leman go neural: Higher-order graph
  neural networks}. In \bibinfo{booktitle}{\emph{Proceedings of the AAAI
  Conference on Artificial Intelligence}}, Vol.~\bibinfo{volume}{33}.
  \bibinfo{pages}{4602--4609}.
\newblock


\bibitem[\protect\citeauthoryear{Narayanan, Chandramohan, Venkatesan, Chen,
  Liu, and Jaiswal}{Narayanan et~al\mbox{.}}{2017a}]%
        {narayanan:mlg2017}
\bibfield{author}{\bibinfo{person}{A. Narayanan}, \bibinfo{person}{M.
  Chandramohan}, \bibinfo{person}{R. Venkatesan}, \bibinfo{person}{L Chen},
  \bibinfo{person}{Y. Liu}, {and} \bibinfo{person}{S. Jaiswal}.}
  \bibinfo{year}{2017}\natexlab{a}.
\newblock \showarticletitle{graph2vec: Learning Distributed Representations of
  Graphs}. In \bibinfo{booktitle}{\emph{13th International Workshop on Mining
  and Learning with Graphs (MLGWorkshop 2017)}}.
\newblock


\bibitem[\protect\citeauthoryear{Narayanan, Chandramohan, Venkatesan, Chen,
  Liu, and Jaiswal}{Narayanan et~al\mbox{.}}{2017b}]%
        {narayanan2017graph2vec}
\bibfield{author}{\bibinfo{person}{Annamalai Narayanan},
  \bibinfo{person}{Mahinthan Chandramohan}, \bibinfo{person}{Rajasekar
  Venkatesan}, \bibinfo{person}{Lihui Chen}, \bibinfo{person}{Yang Liu}, {and}
  \bibinfo{person}{Shantanu Jaiswal}.} \bibinfo{year}{2017}\natexlab{b}.
\newblock \showarticletitle{graph2vec: Learning distributed representations of
  graphs}.
\newblock \bibinfo{journal}{\emph{arXiv preprint arXiv:1707.05005}}
  (\bibinfo{year}{2017}).
\newblock


\bibitem[\protect\citeauthoryear{Neumann, Garnett, Bauckhage, and
  Kersting}{Neumann et~al\mbox{.}}{2016}]%
        {neumann2016propagation}
\bibfield{author}{\bibinfo{person}{Marion Neumann}, \bibinfo{person}{Roman
  Garnett}, \bibinfo{person}{Christian Bauckhage}, {and}
  \bibinfo{person}{Kristian Kersting}.} \bibinfo{year}{2016}\natexlab{}.
\newblock \showarticletitle{Propagation kernels: efficient graph kernels from
  propagated information}.
\newblock \bibinfo{journal}{\emph{Machine Learning}} \bibinfo{volume}{102},
  \bibinfo{number}{2} (\bibinfo{year}{2016}), \bibinfo{pages}{209--245}.
\newblock


\bibitem[\protect\citeauthoryear{Newman}{Newman}{2003}]%
        {newman2003structure}
\bibfield{author}{\bibinfo{person}{Mark~EJ Newman}.}
  \bibinfo{year}{2003}\natexlab{}.
\newblock \showarticletitle{The structure and function of complex networks}.
\newblock \bibinfo{journal}{\emph{SIAM review}} \bibinfo{volume}{45},
  \bibinfo{number}{2} (\bibinfo{year}{2003}), \bibinfo{pages}{167--256}.
\newblock


\bibitem[\protect\citeauthoryear{Niepert, Ahmed, and Kutzkov}{Niepert
  et~al\mbox{.}}{2016}]%
        {niepert2016learning}
\bibfield{author}{\bibinfo{person}{Mathias Niepert}, \bibinfo{person}{Mohamed
  Ahmed}, {and} \bibinfo{person}{Konstantin Kutzkov}.}
  \bibinfo{year}{2016}\natexlab{}.
\newblock \showarticletitle{Learning convolutional neural networks for graphs}.
  In \bibinfo{booktitle}{\emph{International conference on machine learning}}.
  \bibinfo{pages}{2014--2023}.
\newblock


\bibitem[\protect\citeauthoryear{Shervashidze, Schweitzer, Leeuwen, Mehlhorn,
  and Borgwardt}{Shervashidze et~al\mbox{.}}{2011}]%
        {shervashidze2011weisfeiler}
\bibfield{author}{\bibinfo{person}{Nino Shervashidze}, \bibinfo{person}{Pascal
  Schweitzer}, \bibinfo{person}{Erik Jan~van Leeuwen}, \bibinfo{person}{Kurt
  Mehlhorn}, {and} \bibinfo{person}{Karsten~M Borgwardt}.}
  \bibinfo{year}{2011}\natexlab{}.
\newblock \showarticletitle{Weisfeiler-lehman graph kernels}.
\newblock \bibinfo{journal}{\emph{Journal of Machine Learning Research}}
  \bibinfo{volume}{12}, \bibinfo{number}{Sep} (\bibinfo{year}{2011}),
  \bibinfo{pages}{2539--2561}.
\newblock


\bibitem[\protect\citeauthoryear{Shervashidze, Vishwanathan, Petri, Mehlhorn,
  and Borgwardt}{Shervashidze et~al\mbox{.}}{2009}]%
        {shervashidze2009efficient}
\bibfield{author}{\bibinfo{person}{Nino Shervashidze}, \bibinfo{person}{SVN
  Vishwanathan}, \bibinfo{person}{Tobias Petri}, \bibinfo{person}{Kurt
  Mehlhorn}, {and} \bibinfo{person}{Karsten Borgwardt}.}
  \bibinfo{year}{2009}\natexlab{}.
\newblock \showarticletitle{Efficient graphlet kernels for large graph
  comparison}. In \bibinfo{booktitle}{\emph{Artificial Intelligence and
  Statistics}}. \bibinfo{pages}{488--495}.
\newblock


\bibitem[\protect\citeauthoryear{Simonovsky and Komodakis}{Simonovsky and
  Komodakis}{2017}]%
        {simonovsky2017dynamic}
\bibfield{author}{\bibinfo{person}{Martin Simonovsky} {and}
  \bibinfo{person}{Nikos Komodakis}.} \bibinfo{year}{2017}\natexlab{}.
\newblock \showarticletitle{Dynamic edge-conditioned filters in convolutional
  neural networks on graphs}. In \bibinfo{booktitle}{\emph{Proceedings of the
  IEEE conference on computer vision and pattern recognition}}.
  \bibinfo{pages}{3693--3702}.
\newblock


\bibitem[\protect\citeauthoryear{Sun, Hoffman, Verma, and Tang}{Sun
  et~al\mbox{.}}{2020}]%
        {Sun2020InfoGraph}
\bibfield{author}{\bibinfo{person}{Fan-Yun Sun}, \bibinfo{person}{Jordan
  Hoffman}, \bibinfo{person}{Vikas Verma}, {and} \bibinfo{person}{Jian Tang}.}
  \bibinfo{year}{2020}\natexlab{}.
\newblock \showarticletitle{InfoGraph: Unsupervised and Semi-supervised
  Graph-Level Representation Learning via Mutual Information Maximization}. In
  \bibinfo{booktitle}{\emph{International Conference on Learning
  Representations}}.
\newblock
\urldef\tempurl%
\url{https://openreview.net/forum?id=r1lfF2NYvH}
\showURL{%
\tempurl}


\bibitem[\protect\citeauthoryear{Sun, Ajwani, Nicholson, Sala, and
  Parthasarathy}{Sun et~al\mbox{.}}{2017}]%
        {sun2017breaking}
\bibfield{author}{\bibinfo{person}{Jiankai Sun}, \bibinfo{person}{Deepak
  Ajwani}, \bibinfo{person}{Patrick~K Nicholson}, \bibinfo{person}{Alessandra
  Sala}, {and} \bibinfo{person}{Srinivasan Parthasarathy}.}
  \bibinfo{year}{2017}\natexlab{}.
\newblock \showarticletitle{Breaking cycles in noisy hierarchies}. In
  \bibinfo{booktitle}{\emph{Proceedings of the 2017 ACM on Web Science
  Conference}}. \bibinfo{pages}{151--160}.
\newblock


\bibitem[\protect\citeauthoryear{Tzeng and Wu}{Tzeng and Wu}{2019}]%
        {tzeng2019distributed}
\bibfield{author}{\bibinfo{person}{Ruo-Chun Tzeng} {and}
  \bibinfo{person}{Shan-Hung Wu}.} \bibinfo{year}{2019}\natexlab{}.
\newblock \showarticletitle{Distributed, Egocentric Representations of Graphs
  for Detecting Critical Structures}. In
  \bibinfo{booktitle}{\emph{International Conference on Machine Learning}}.
  \bibinfo{pages}{6354--6362}.
\newblock


\bibitem[\protect\citeauthoryear{van~der Maaten and Hinton}{van~der Maaten and
  Hinton}{2008}]%
        {vanDerMaaten2008}
\bibfield{author}{\bibinfo{person}{Laurens van~der Maaten} {and}
  \bibinfo{person}{Geoffrey Hinton}.} \bibinfo{year}{2008}\natexlab{}.
\newblock \showarticletitle{Visualizing Data using {t-SNE}}.
\newblock \bibinfo{journal}{\emph{Journal of Machine Learning Research}}
  \bibinfo{volume}{9} (\bibinfo{year}{2008}), \bibinfo{pages}{2579--2605}.
\newblock
\urldef\tempurl%
\url{http://www.jmlr.org/papers/v9/vandermaaten08a.html}
\showURL{%
\tempurl}


\bibitem[\protect\citeauthoryear{Vaswani, Shazeer, Parmar, Uszkoreit, Jones,
  Gomez, Kaiser, and Polosukhin}{Vaswani et~al\mbox{.}}{2017}]%
        {vaswani2017attention}
\bibfield{author}{\bibinfo{person}{Ashish Vaswani}, \bibinfo{person}{Noam
  Shazeer}, \bibinfo{person}{Niki Parmar}, \bibinfo{person}{Jakob Uszkoreit},
  \bibinfo{person}{Llion Jones}, \bibinfo{person}{Aidan~N Gomez},
  \bibinfo{person}{{\L}ukasz Kaiser}, {and} \bibinfo{person}{Illia
  Polosukhin}.} \bibinfo{year}{2017}\natexlab{}.
\newblock \showarticletitle{Attention is all you need}. In
  \bibinfo{booktitle}{\emph{Advances in neural information processing
  systems}}. \bibinfo{pages}{5998--6008}.
\newblock


\bibitem[\protect\citeauthoryear{Veli{\v{c}}kovi{\'c}, Cucurull, Casanova,
  Romero, Lio, and Bengio}{Veli{\v{c}}kovi{\'c} et~al\mbox{.}}{2018}]%
        {velivckovic2017graph}
\bibfield{author}{\bibinfo{person}{Petar Veli{\v{c}}kovi{\'c}},
  \bibinfo{person}{Guillem Cucurull}, \bibinfo{person}{Arantxa Casanova},
  \bibinfo{person}{Adriana Romero}, \bibinfo{person}{Pietro Lio}, {and}
  \bibinfo{person}{Yoshua Bengio}.} \bibinfo{year}{2018}\natexlab{}.
\newblock \showarticletitle{Graph attention networks}. In
  \bibinfo{booktitle}{\emph{International Conference on Learning
  Representations}}.
\newblock
\urldef\tempurl%
\url{https://openreview.net/forum?id=rJXMpikCZ}
\showURL{%
\tempurl}


\bibitem[\protect\citeauthoryear{Vishwanathan, Schraudolph, Kondor, and
  Borgwardt}{Vishwanathan et~al\mbox{.}}{2010}]%
        {vishwanathan2010graph}
\bibfield{author}{\bibinfo{person}{S~Vichy~N Vishwanathan},
  \bibinfo{person}{Nicol~N Schraudolph}, \bibinfo{person}{Risi Kondor}, {and}
  \bibinfo{person}{Karsten~M Borgwardt}.} \bibinfo{year}{2010}\natexlab{}.
\newblock \showarticletitle{Graph kernels}.
\newblock \bibinfo{journal}{\emph{Journal of Machine Learning Research}}
  \bibinfo{volume}{11}, \bibinfo{number}{Apr} (\bibinfo{year}{2010}),
  \bibinfo{pages}{1201--1242}.
\newblock


\bibitem[\protect\citeauthoryear{Wu, Pan, Chen, Long, Zhang, and Yu}{Wu
  et~al\mbox{.}}{2019}]%
        {wu2019comprehensive}
\bibfield{author}{\bibinfo{person}{Zonghan Wu}, \bibinfo{person}{Shirui Pan},
  \bibinfo{person}{Fengwen Chen}, \bibinfo{person}{Guodong Long},
  \bibinfo{person}{Chengqi Zhang}, {and} \bibinfo{person}{Philip~S Yu}.}
  \bibinfo{year}{2019}\natexlab{}.
\newblock \showarticletitle{A comprehensive survey on graph neural networks}.
\newblock \bibinfo{journal}{\emph{arXiv preprint arXiv:1901.00596}}
  (\bibinfo{year}{2019}).
\newblock


\bibitem[\protect\citeauthoryear{Xu, Hu, Leskovec, and Jegelka}{Xu
  et~al\mbox{.}}{2019}]%
        {xu2018how}
\bibfield{author}{\bibinfo{person}{Keyulu Xu}, \bibinfo{person}{Weihua Hu},
  \bibinfo{person}{Jure Leskovec}, {and} \bibinfo{person}{Stefanie Jegelka}.}
  \bibinfo{year}{2019}\natexlab{}.
\newblock \showarticletitle{How Powerful are Graph Neural Networks?}. In
  \bibinfo{booktitle}{\emph{International Conference on Learning
  Representations}}.
\newblock
\urldef\tempurl%
\url{https://openreview.net/forum?id=ryGs6iA5Km}
\showURL{%
\tempurl}


\bibitem[\protect\citeauthoryear{Xu, Li, Tian, Sonobe, Kawarabayashi, and
  Jegelka}{Xu et~al\mbox{.}}{2018}]%
        {xu2018representation}
\bibfield{author}{\bibinfo{person}{Keyulu Xu}, \bibinfo{person}{Chengtao Li},
  \bibinfo{person}{Yonglong Tian}, \bibinfo{person}{Tomohiro Sonobe},
  \bibinfo{person}{Ken-ichi Kawarabayashi}, {and} \bibinfo{person}{Stefanie
  Jegelka}.} \bibinfo{year}{2018}\natexlab{}.
\newblock \showarticletitle{Representation Learning on Graphs with Jumping
  Knowledge Networks}. In \bibinfo{booktitle}{\emph{International Conference on
  Machine Learning}}. \bibinfo{pages}{5449--5458}.
\newblock


\bibitem[\protect\citeauthoryear{Yanardag and Vishwanathan}{Yanardag and
  Vishwanathan}{2015}]%
        {yanardag2015deep}
\bibfield{author}{\bibinfo{person}{Pinar Yanardag} {and} \bibinfo{person}{SVN
  Vishwanathan}.} \bibinfo{year}{2015}\natexlab{}.
\newblock \showarticletitle{Deep graph kernels}. In
  \bibinfo{booktitle}{\emph{Proceedings of the 21th ACM SIGKDD International
  Conference on Knowledge Discovery and Data Mining}}. ACM,
  \bibinfo{pages}{1365--1374}.
\newblock


\bibitem[\protect\citeauthoryear{Yang, Wang, Song, Yuan, and Tao}{Yang
  et~al\mbox{.}}{2019}]%
        {yang2019spagan}
\bibfield{author}{\bibinfo{person}{Yiding Yang}, \bibinfo{person}{Xinchao
  Wang}, \bibinfo{person}{Mingli Song}, \bibinfo{person}{Junsong Yuan}, {and}
  \bibinfo{person}{Dacheng Tao}.} \bibinfo{year}{2019}\natexlab{}.
\newblock \showarticletitle{SPAGAN: shortest path graph attention network}. In
  \bibinfo{booktitle}{\emph{Proceedings of the 28th International Joint
  Conference on Artificial Intelligence}}. AAAI Press,
  \bibinfo{pages}{4099--4105}.
\newblock


\bibitem[\protect\citeauthoryear{Ying, You, Morris, Ren, Hamilton, and
  Leskovec}{Ying et~al\mbox{.}}{2018}]%
        {ying2018hierarchical}
\bibfield{author}{\bibinfo{person}{Zhitao Ying}, \bibinfo{person}{Jiaxuan You},
  \bibinfo{person}{Christopher Morris}, \bibinfo{person}{Xiang Ren},
  \bibinfo{person}{Will Hamilton}, {and} \bibinfo{person}{Jure Leskovec}.}
  \bibinfo{year}{2018}\natexlab{}.
\newblock \showarticletitle{Hierarchical graph representation learning with
  differentiable pooling}. In \bibinfo{booktitle}{\emph{Advances in Neural
  Information Processing Systems}}. \bibinfo{pages}{4800--4810}.
\newblock


\bibitem[\protect\citeauthoryear{You, Ying, and Leskovec}{You
  et~al\mbox{.}}{2019}]%
        {you2019position}
\bibfield{author}{\bibinfo{person}{Jiaxuan You}, \bibinfo{person}{Rex Ying},
  {and} \bibinfo{person}{Jure Leskovec}.} \bibinfo{year}{2019}\natexlab{}.
\newblock \showarticletitle{Position-aware Graph Neural Networks}. In
  \bibinfo{booktitle}{\emph{International Conference on Machine Learning}}.
  \bibinfo{pages}{7134--7143}.
\newblock


\bibitem[\protect\citeauthoryear{Zhang, Cui, Neumann, and Chen}{Zhang
  et~al\mbox{.}}{2018}]%
        {zhang2018end}
\bibfield{author}{\bibinfo{person}{Muhan Zhang}, \bibinfo{person}{Zhicheng
  Cui}, \bibinfo{person}{Marion Neumann}, {and} \bibinfo{person}{Yixin Chen}.}
  \bibinfo{year}{2018}\natexlab{}.
\newblock \showarticletitle{An end-to-end deep learning architecture for graph
  classification}. In \bibinfo{booktitle}{\emph{Thirty-Second AAAI Conference
  on Artificial Intelligence}}.
\newblock


\end{thebibliography}

%\appendix


\end{document}